%% file: main.tex
\documentclass[11pt]{article}
\usepackage{booktabs} 
\usepackage{graphicx}
\usepackage{authblk}
\usepackage[top=1in, bottom=1in, left=1in, right=1in]{geometry}

\renewcommand{\baselinestretch}{1.15} 

\input{macro}

\begin{document}

\begin{titlepage}


\title{The Surprising Benefits of Base Rate Neglect in Robust Aggregation\thanks{This work is supported by National Science and Technology Major Project (2022ZD0114904). We thank Tracy Xiao Liu for stimulating comments and suggestions.}}


\author[1]{Yuqing Kong}
\author[2]{Shu Wang}
\author[1]{Ying Wang\thanks{Authors contributed equally and are listed in alphabetical order. }}

\affil[1]{Center on Frontiers of Computing Studies, School of Computer Science, Peking Univeristy}
\affil[2]{Department of Economics, School of Economics and Management, Tsinghua Univeristy}
\affil[1]{\texttt{\{yuqing.kong, wying2000\}@pku.edu.cn}}
\affil[2]{\texttt{shu-wang20@mails.tsinghua.edu.cn}}
\date{}

\maketitle


\begin{abstract}
Robust aggregation integrates predictions from multiple experts without knowledge of the experts' information structures. Prior work assumes experts are Bayesian, providing predictions as perfect posteriors based on their signals. However, real-world experts often deviate systematically from Bayesian reasoning. Our work considers experts who tend to ignore the base rate. We find that a certain degree of base rate neglect helps with robust forecast aggregation.

Specifically, we consider a forecast aggregation problem with two experts who each predict a binary world state after observing private signals. Unlike previous work, we model experts exhibiting base rate neglect, where they incorporate the base rate information to degree $\lambda\in[0,1]$, with $\lambda=0$ indicating complete ignorance and $\lambda=1$ perfect Bayesian updating. To evaluate aggregators' performance, we adopt \citet{arieli2018robust}'s worst-case regret model, which measures the maximum regret across the set of considered information structures compared to an omniscient benchmark. Our results reveal the surprising V-shape of regret as a function of $\lambda$. 
That is, predictions with an intermediate incorporating degree of base rate $\lambda<1$ can counter-intuitively lead to lower regret than perfect Bayesian posteriors with $\lambda=1$.
We additionally propose a new aggregator with low regret robust to unknown $\lambda$. Finally, we conduct an empirical study to test the base rate neglect model and evaluate the performance of various aggregators\footnote{The data collected in the empirical study is available at \href{https://github.com/EconCSPKU/Probability-Task-Data}{https://github.com/EconCSPKU/Probability-Task-Data}.}. 

\end{abstract}

\setcounter{tocdepth}{2} 

\end{titlepage}

\input{body}
\clearpage
\bibliographystyle{ACM-Reference-Format}
\bibliography{reference}

\setcounter{table}{0}
\renewcommand{\thetable}{A\arabic{table}}

\begin{toappendix}
\input{appendix_experiment}

\setcounter{table}{0}
\renewcommand{\thetable}{G\arabic{table}}

\input{additional_analysis}
\end{toappendix}

\end{document}

%% file: macro.tex
\usepackage{xpatch}
\usepackage{pgfplots}
\pgfplotsset{compat=1.18}
\usepackage{pgfmath}
\usepackage{apxproof}
\usepackage{amsthm}
\usepackage{threeparttable}
\usepackage{makecell}
\usepackage{multirow}
\usepackage{cancel}
\usepackage{ulem}
\usepackage{subcaption}
\usepackage{natbib}
\usepackage{hyperref}
\usepackage{amsmath}
\usepackage{cleveref}
\usepackage{nicefrac}       
\usepackage{amsfonts}       
\newtheorem{theorem}{Theorem}

\newtheorem{lemma}{Lemma}
\newtheorem{claim}{Claim}
\newtheorem{observation}{Observation}

\hypersetup{
	colorlinks   = true, 
	urlcolor     = blue, 
	linkcolor    = blue, 
	citecolor   = blue    
}

\newcommand{\hide}[1]{} 

 \newcommand\shu[1]{\textcolor{red}{[Shu: #1]}}
 \newcommand\yk[1]{}

\newcommand\bayes[0]{x^{Bayes}}
\newcommand\brn[0]{x^{BRN}}
\newcommand\logit[0]{\mathrm{logit}}

\definecolor{blue1}{RGB}{186,228,188}
\definecolor{blue2}{RGB}{123,204,196}
\definecolor{blue3}{RGB}{67,162,202}
\definecolor{blue4}{RGB}{8,104,172}
\definecolor{pink1}{RGB}{251,180,185}
\definecolor{pink2}{RGB}{247,104,161}
\definecolor{pink3}{RGB}{197,27,138}
\definecolor{pink4}{RGB}{122,1,119}

%% file: body.tex
\input{1_intro}

\input{2_problem_statement}

\input{3_omniscient_aggregator}
\section{V-shape of Regret Curves}
\input{4_Regret_Curve_Theory}

\input{5_Lower_Bound}

\section{Numerical Results}\label{section: numerical results}
In this section, we present several numerical results about the regret of specific aggregators. These regret curves are all single-troughed as our theoretical result in Theorem~\ref{theorem: V-shaped upper bound}. Each of them provides an upper bound for the optimal regret $\inf_{g}{R_\lambda(g)}$, whose lower bound is studied in Theorem~\ref{theorem: unimodality of the lower bound}. Here is an outline of our results:
\begin{itemize}
    \item [(1)] Average Prior is V-shaped: While the regret of the simple average aggregator monotonically decreases as the value of $\lambda$ increases, interestingly, we find the average prior aggregator achieves the lowest regret with $\lambda < 1$. 
    \item [(2)] New Family of Aggregators: We identify a family of aggregators, $\{f^{\hat{\lambda}}_{ap}\}_{0\le \lambda \le 1}$, named $\hat{\lambda}$-base rate balancing aggregators. The minimum regret of these aggregators closely approaches our constructed lower bound $lb(\lambda)$, with a small error margin below 0.003. 
    Since these regrets are all upper bounds of $\inf_{g}R_\lambda(g)$, our finding indicates that our proposed lower bound is almost tight. 
    \item [(3)] Almost-zero\hide{$\approx 0$} Regret at $\lambda=0.5$: There exists an aggregator $f^{0.5}_{ap}$ that achieves almost-zero regret when the prior consideration degree $\lambda$ is one-half. 

    \item [(4)] Nearly Optimal Aggregator across All $\lambda$: A particular $\hat{\lambda}$-base rate balancing aggregator, $f_{ap}^{0.7}$, performs well across different $\lambda$. This robust aggregator is nearly optimal within a 0.013 loss compared to the optimal aggregator for all $\lambda$.
\end{itemize}




\subsection{Regret of Existing Aggregators}
We first evaluate the following two aggregators numerically \footnote{We employ the same method in \citet{arieli2018robust}, the global optimization box in Matlab.}. 
\begin{itemize}
    \item Simple Average Aggregator: \(f_{ave}(x_1, x_2) = \frac{x_1+x_2}{2}.\)
    \item Average Prior Aggregator: \(f_{ap}(x_1, x_2) = \frac{(1-\hat{\mu})x_1x_2}{(1-\hat{\mu})x_1x_2 + \hat{\mu}(1-x_1)(1-x_2)},\) where symbol $\hat{\mu}$ is the prior proxy set to $\frac{x_1 + x_2}{2}$. 
\end{itemize}

Figure~\ref{fig: main results} presents the numerical regret curves $R_\lambda(f_{ave})$, $R_\lambda(f_{ap})$, and the lower bound curve $lb(\lambda)$, considering $\lambda$ as a multiple of a tenth. As we can see, the regret of the simple average aggregator, $R_\lambda(f_{ave})$, initially decreases and then stabilizes. Notably, the average prior aggregator achieves a lower regret at some interior point where $\lambda < 1$, suggesting that the regret curve $R_\lambda(f_{ap})$ is V-shaped. This observation is somewhat counterintuitive --- when the experts incorrectly lower the prior weight and make wrong predictions, the aggregation results, however, turn out to be better.

In addition, as shown in this figure, the regret curve of average prior closely approaches the lower bound when $\lambda$
is close to $1$. This implies that the average prior is a nearly optimal aggregator for experts who are perfect Bayesian. 

\subsection{Nearly Optimal Aggregators for Various Degrees $\lambda$}
As shown in Figure~\ref{fig: main results}, there remains a large gap between the lower bound $lb(\lambda)$ and the regret curves of existing aggregators when degree $\lambda$ is small. This gap indicates the poor performance of these aggregators when experts demonstrate a considerable tendency of base rate neglect.

To better aggregate predictions from various expert groups, for each $\lambda$, we require a nearly optimal aggregator. We propose a new family of aggregators, the \textbf{$\hat{\mathbf{\lambda}}$-base rate balancing aggregators}, denoted as $\{f^{\hat{\lambda}}_{ap}\}_{0\le \hat{\lambda} \le 1}$. 
Formally, we define the aggregator $f^{\hat{\lambda}}_{ap}$ as \[f^{\hat{\lambda}}_{ap}(x_1, x_2) = \frac{(1-\hat{\mu})^{2\hat{\lambda} - 1}x_1x_2}{(1-\hat{\mu})^{2\hat{\lambda} - 1}x_1x_2 + \hat{\mu}^{2\hat{\lambda} - 1}(1-x_1)(1-x_2)}, \] where the prior proxy $\hat{\mu}$ is set to the average prediction of experts, i.e., $\frac{x_1+x_2}{2}$. The $\hat{\lambda}$-base rate balancing aggregators include average prior aggregator as a special case where $\hat{\lambda} = 1$, i.e., $f_{ap}^{1} = f_{ap}$. 

These aggregators adopt the same aggregation methodology as the omniscient aggregator (see Observation~\ref{lemma: omniscient aggregator}). However, unlike the omniscient aggregator, the $\hat{\lambda}$-base rate balancing aggregator lacks knowledge of the true prior $\mu$ and the prior consideration degree $\lambda$. Instead, these aggregators embed a fixed value $\hat{\lambda}$ within the aggregation formula and use the average prediction as an estimate for the actual prior. Particularly, even if the embedded value $\hat{\lambda}$ is exactly the degree $\lambda$, 
the difference between prior proxy $\hat{\mu}$ and true prior $\mu$ leads to a non-zero regret. 
For example, even when all the experts are Bayesian, it remains a gap between the average prior (i.e., $f_{ap} = f_{ap}^{1}$) and the omniscient aggregator because the average prediction $\frac{x_1 + x_2}{2}$ does not always meet the actual prior $\mu$.


\begin{figure}
    \centering
    \begin{tikzpicture}[scale = 0.85]
    \begin{axis}[
        width=1\textwidth, 
        height=7.8cm, 
        xlabel={base rate consideration degree $\lambda$},
        ylabel={regret $R_\lambda(f)$},
        xmin=0, xmax=1,
        ymin=0,
        ytick={0.05,0.1,0.15,0.2,0.25}, 
        yticklabels={0.05,0.10,0.15,0.20,0.25} , 
        legend pos=north east,
        grid style=dashed,
        ]
    
    \addplot[color=lightBlue,
    solid,
    line width=1pt,
    mark=o,
    mark options={solid},
    mark size=2.5pt] coordinates {
        (0,0.25) (0.1,0.100347) (0.2,0.042957) (0.3,0.015106) (0.4,0.003017) (0.5,0) (0.6,0.001939) (0.7,0.006288) (0.8,0.011604) (0.9,0.017143) (1,0.022542)
    };
    \addlegendentry{lower bound $lb(\lambda)$}
    
    \addplot[color=darkGreen,
    solid,
    line width=1pt,
    mark=triangle,
    mark options={solid},
    mark size=2.5pt] coordinates {
        (0,0.249999) (0.1,0.100717) (0.2,0.043769) (0.3,0.015889) (0.4,0.003341) (0.5,0) (0.6,0.002482) (0.7,0.008698) (0.8,0.017282) (0.9,0.027317) (1,0.038177)
    };
    \addlegendentry{$\hat{\lambda}$-base rate balancing aggregator with $\hat{\lambda} = 0.5$}
    
    \addplot[color=darkBlue,
    solid,
    line width=1pt,
    mark=square,
    mark options={solid},
    mark size=2.5pt
    ] coordinates {
        (0,0.249999) (0.1,0.112599) (0.2,0.054847) (0.3,0.02508) (0.4,0.009709) (0.5,0.002478) (0.6,0.002237) (0.7,0.006743) (0.8,0.013867) (0.9,0.022527) (1,0.032162)
    };
    \addlegendentry{$\hat{\lambda}$-base rate balancing aggregator with $\hat{\lambda} = 0.7$}

    \addplot[color=lightGreen,
    solid,
    line width=1pt,
    mark=diamond,
    mark options={solid},
    mark size=2.5pt] coordinates {
        (0,0.25) (0.1,0.147781) (0.2,0.094201) (0.3,0.060114) (0.4,0.037578) (0.5,0.022542) (0.6,0.015458) (0.7,0.014549) (0.8,0.014352) (0.9,0.018257) (1,0.025992)
    };
    \addlegendentry{$\hat{\lambda}$-base rate balancing aggregator with $\hat{\lambda} = 1$}
    \end{axis}
    \end{tikzpicture}
    \caption{the Regret of $\hat{\lambda}$-base rate balancing aggregators $f^\lambda_{ap}$}
    \label{fig: EAP_regret}
\end{figure}

The regret curves of $\hat{\lambda}$-base rate balancing aggregators 
are shown in Figure~\ref{fig: EAP_regret}.
When the embedded parameter $\hat{\lambda}$ is set to $0.5$, the numerical regret $R_\lambda(f^{0.5}_{ap})$ closely approaches the lower bound for cases where $\lambda \le 0.5$, implying its near-optimality when experts slightly incorporate the prior into their predictions. We highlight that this aggregator achieves almost-zero regret for $\lambda = 0.5$, i.e. $R_{0.5}(f^{0.5}_{ap}) \approx 0$. This surprising finding implies the negligible distortion between the partial information contained in experts' predictions and the full information that the omniscient aggregator can access at $\lambda = 0.5$. In other words, when experts integrate their prior knowledge at a degree of $\lambda = 0.5$, a decision-maker without specific knowledge of the underlying information structure can effectively approximate the Bayesian aggregator's posterior, by solely relying on experts' predictions. 


The regret of $\hat{\lambda}$-base rate balancing aggregator with $\hat{\lambda} = 0.5$ and that with $\hat{\lambda} = 1$ together form an upper bound of the optimal regret $\inf_{g}{R_\lambda(g)}$, which is notably close to the previously established lower bound $lb(\lambda)$, with a small error margin up to 0.003. Notably, for degree $\lambda \le 0.8$, this error remains exceptionally low (not exceeding 0.001).
Such proximity between upper bound, i.e., $\min\{R_\lambda(f_{ap}^{0.5}), R_\lambda(f_{ap}^{1})\}$, and lower bound, i.e., $lb(\lambda)$, suggests that both of them are almost tight.

\subsection{Robust Aggregator for Unknown Degree $\lambda$}
The aforementioned nearly optimal aggregators help aggregation when the prior consideration degree $\lambda$ is known. However, the decision maker generally does not know to what extent the experts consider the prior. Noticing that a nearly optimal aggregator at degree $\lambda_1$ may poorly perform at another degree $\lambda_2$, we require a robust aggregator that aggregates predictions effectively across different $\lambda$ values. 


We employ a new framework as mentioned in Introduction and evaluate the performance of an aggregator by assessing the overall regret defined in Equation (\ref{def: overall regret}). 
This overall regret is hard to compute due to the complexity of deciding the optimal regret $\inf_g R_\lambda(g)$. Instead, we use the regret lower bound $lb(\lambda)$ to replace the optimal regret, providing an upper bound for the overall regret $R(f)$, denoted as $\hat{R}(f)$. Formally,
\[\hat{R}(f) = \sup_{\lambda \in [0, 1]}\{R_\lambda(f) - lb(\lambda)\} \ge R(f).\]

Table~\ref{tab: numerical results of new criterion} shows the numerical results for this upper bound of regret. We find that with $\hat{\lambda} = 0.7$, the $\hat{\lambda}$-base rate balancing aggregator attains an aggregated outcome with a regret below 0.013, irrespective of experts' prior consideration degree $\lambda$.

\begin{table}
    \centering
    \begin{tabular}{ccccc}
    \hline \hline
       aggregator  &  $f_{ave}$ & $f_{ap}$ & $f_{ap}^{0.5}$ & $f_{ap}^{0.7}$\\
       \hline
       $\hat{R}(f)$  & 0.062 & $\ge$ 0.051 & 0.015 &$\approx$ 0.013 \\
    \hline \hline
    \end{tabular}
    \caption{Numerical Results of $\hat{R}(f)$}
    \label{tab: numerical results of new criterion}
\end{table}

\input{7_Experiment}

\input{8_Conclusion}

%% file: 1_intro.tex
\section{Introduction}



Meet Jane --- a generally healthy woman who has been feeling under the weather lately. She decides to get checked out by two doctors to see if she has a particular disease that's been going around. Doctor A runs a diagnostic test and tells Jane there's a 70\% chance she has the disease. Meanwhile, Doctor B runs a different diagnostic test and tells Jane her chance is 60\%. Jane wonders how should she combine these two assessments to understand her overall likelihood of having this disease. 

If the doctors were perfect Bayesians, Jane could combine the results using her knowledge of the disease's 15\% prevalence rate in the general population. But she may not know the prevalence rate. More importantly, in the real world, doctors may not be perfect Bayesians.

Say you're Doctor A. You know this disease affects 15\% of the population in general, and your test is 80\% accurate at detecting it. If Jane tests positive, what is the chance she has the disease? An intuitive response is 80\% --- after all, that's what your test accuracy is. A slightly more informed guess might be 70\%. But using the Bayesian rule, the actual chance Jane has the disease is only 41\%! 

Doctor A's example is an adaptation of the famous taxicab problem. Most people will answer 80\% whereas the correct answer is 41\%. \citet{kahneman1973psychology} used this example to illustrate the prevalent cognitive bias of humans termed base rate neglect (or base rate fallacy), where people tend to ignore the base rate and instead focus on new information. 


This raises an important research question: How should patients like Jane aggregate medical opinions when doctors may exhibit base rate fallacy and the true prevalence of the disease is unknown? This question is faced in many other decision-making situations. For example, a business leader might get a few different guesses about next quarter's sales from analysts. The analysts might not look enough at older sales data. Also, a government official could get some predictions about how far an epidemic will spread. The experts might ignore past rates. In the machine learning context, a decision-maker elicits forecasts from data scientists. The data scientists may over-rely on a machine's prediction and ignore the true prior\footnote{\href{https://cacm.acm.org/blogs/blog-cacm/262443-the-base-rate-neglect-cognitive-bias-in-data-science/fulltext}{https://cacm.acm.org/blogs/blog-cacm/262443-the-base-rate-neglect-cognitive-bias-in-data-science/fulltext}}.

To address the question, we consider a model with experts who exhibit base rate neglect. The experts share a base rate $\mu=\Pr[\omega=1]$. Each expert $i$ also knows the relationship between signal $s_i$ (e.g. medical test result) and the binary world state $\omega$. However, rather than generating a Bayesian posterior, she only partially incorporates the prior $\mu$ into her evaluation $x_i$ of the truth.

\paragraph{\textbf{The Base Rate Neglect Model}} The extent to which the prior is considered is quantified by a parameter $\lambda\in [0,1]$. We name this parameter as the prior consideration degree (or base rate consideration degree). When $\lambda=0$, the expert completely ignores the prior and reports $x_i = \frac{\Pr[S_i = s_i | \omega=1]}{\Pr[S_i = s_i | \omega=1] + \Pr[S_i = s_i | \omega=0]}$. For example, if a medical test is positive, an expert with $\lambda=0$ would report the test's accuracy rather than incorporating the rarity of the disease. As $\lambda$ increases, the expert puts more weight on the prior when forming her posterior evaluation $x_i = \brn_i(s_i,\lambda)$ where \begin{align*}
\brn_i(s_i, \lambda) &= \frac{\mu^{\lambda}\cdot \Pr[S_i = s_i\mid \omega = 1]}{\mu^{\lambda}\cdot \Pr[S_i = s_i\mid \omega = 1]+(1-\mu)^{\lambda}\cdot \Pr[S_i = s_i\mid \omega = 0]}.
\end{align*}

Let $\bayes_i(s_i)$ denote the Bayesian posterior of expert $i$ upon receiving signal $s_i$.
We also have 

\begin{observation}\label{lemma: BNR prediction vs Bayesian}


\begin{align*}
\brn_i(s_i, \lambda) = &\frac{(1-\mu)^{1-\lambda}\cdot \bayes(s_i) }{(1-\mu)^{1-\lambda}\cdot \bayes(s_i)+\mu^{1-\lambda}\cdot (1-\bayes(s_i))}.
\end{align*} 

It induces a linear relationship between the log odds
\[ \logit (\brn_i(s_i, \lambda))=\logit(\bayes(s_i))-(1-\lambda)\logit(\mu)  \]
where $\logit(p)=\log\frac{p}{1-p}$.
\end{observation}

When $\lambda=1$, the expert becomes a perfect Bayesian, i.e., $\brn_i(s_i, 1) = \bayes_i(s_i)$, properly integrating the prior and signal likelihood. We adopt the above model \cite{benjamin2019base} because the prior experimental studies such as \citet{grether1992testing} have demonstrated base rate neglect by fitting a linear relationship between log odds and finding $\lambda<1$.

\paragraph{\textbf{Robust Framework}} We focus on the two-expert case. To integrate experts' evaluations, we use an aggregator $f:[0,1]^2 \to [0,1]$ which inputs evaluations and generates an aggregated forecast. The aggregator lacks knowledge of the information structure --- the joint distribution over signals and the state. To evaluate the performance of this aggregator, we follow the robust framework of \citet{arieli2018robust}. In this framework, an omniscient aggregator is compared to assess the loss of $f$. The omniscient aggregator knows the information structure and signals and outputs the Bayesian aggregator's posterior given all experts' signals. The regret of the aggregator is calculated as the worst-case relative loss of aggregator $f$, where the worst-case refers to the worst information structure that maximizes the relative loss of $f$.

\paragraph{\textbf{A New Framework under Base Rate Neglect}} This paper follows the above regret definition but replaces perfect Bayesian experts with experts who consider the prior information to degree $\lambda$. This leads to a new regret definition $R_\lambda(f)$ for each $\lambda$, and generalizes the regret in \citet{arieli2018robust} whose regret corresponds to $R_{\lambda=1}(f)$.

Recognizing that the aggregator generally lacks information about degree $\lambda$, we introduce a new criterion to measure the regret of an aggregator $f$ under this uncertainty:

\begin{equation}\label{def: overall regret}
R(f) = \sup_{\lambda \in [0, 1]}\{R_\lambda(f) - \inf_{g}{R_\lambda(g)}\}
\end{equation}

The overall regret $R(f)$ is defined as the maximum regret over all $\lambda$ compared to the optimal aggregator for that $\lambda$. An aggregator with low $R(f)$ would perform well across different consideration degrees of base rate, rather than relying on a specific assumption about $\lambda$. 

\subsection{Summary of Results}

We focus on the setting of two experts and conditionally independent information structures. That is conditioning on the true state $\omega$, two experts' signals are independent. For general structures, \citet{arieli2018robust} prove a negative result of effective aggregation. The negative result still holds in our scenario\footnote{We defer the detailed explanation in Appendix ~\ref{app:claim}.}. In the conditionally independent setting, we obtain the following results.

\begin{toappendix}
\begin{claim}
For any prior consideration degree $\lambda$, no aggregator can reach a regret less than 0.25 in general information structures.
\end{claim}
\begin{proof}
\label{app:claim}
\begin{table}[h]
    \centering
    \begin{subtable}{.4\linewidth} 
        \centering
       \begin{tabular}{ccc}
            & $S_2 = r$  & $S_2 = b$\\
        \cline{2-3}
          \multicolumn{1}{c|}{$S_1 = r$}   & \multicolumn{1}{c|}{$\frac14$}   & \multicolumn{1}{c|}{0}  \\
        \cline{2-3}
           \multicolumn{1}{c|}{$S_1 = b$}   &\multicolumn{1}{c|}{0} &   \multicolumn{1}{c|}{$\frac14$} \\
           \cline{2-3}
        \end{tabular}
        \caption*{$\omega = 0$}
    \end{subtable}%
    \begin{subtable}{.4\linewidth} 
        \centering
        \begin{tabular}{ccc}
            & $S_2 = r$ & $S_2 = b$\\
           \cline{2-3}
          \multicolumn{1}{c|}{$S_1 = r$}   & \multicolumn{1}{c|}{0} & \multicolumn{1}{c|}{$\frac14$}\\
           \cline{2-3}
           \multicolumn{1}{c|}{$S_1 = b$} & \multicolumn{1}{c|}{$\frac14$} & \multicolumn{1}{c|}{0}\\
                   \cline{2-3}
        \end{tabular}
        \caption*{$\omega = 1$}
    \end{subtable}
    \caption{Joint Distribution of the General Information Structure}
    \label{tab: general structure instance}
\end{table}
Consider a general information structure $\theta$ where $\mathcal{S}_1 = \mathcal{S}_2 = \{r,b\}$ and the joint distribution of states and signals is specified in Table~\ref{tab: general structure instance}.

In this setup, the signals for both experts are independent and uniformly drawn from the signal space.  The determination of the world state $\omega$ is based on the combination of received signals: $\omega = 0$ when both experts receive the same signal (either both $r$ or both $b$), and $\omega = 1$ when their signals differ.

Given this structure, regardless of the prior consideration degree $\lambda$ or the specific signal received, each expert will predict $\frac{1}{2}$. In such case, an ignorant aggregator can at best give an aggregated result as $\frac{1}{2}$. However, the omniscient aggregator, which has complete knowledge of the experts' signals, can accurately deduce the actual world state from the experts' signals, resulting in a relative loss of at least $0.25$.

Therefore, for any aggregator $f$ and any degree $\lambda$, $R_\lambda(f) \ge 0.25$ holds for general information structures.
\end{proof}
\end{toappendix}



\input{color_table}
\begin{figure}
    \centering
    \begin{tikzpicture}[scale = 0.85]
    \begin{axis}[
        width=1\textwidth, 
        height=7.8cm, 
        xlabel={base rate consideration degree $\lambda$},
        ylabel={regret $R_\lambda(f)$},
        xmin=0, xmax=1,
        ymin=0,
        ytick={0.05,0.1,0.15,0.2,0.25}, 
        yticklabels={0.05,0.10,0.15,0.20,0.25} , 
        legend pos=north east,
        grid style=dashed,
        ]
    
    \addplot[color=lightBlue,
    solid,
    line width=1pt,
    mark=o,
    mark options={solid},
    mark size=2.5pt] coordinates {
        (0,0.25) (0.1,0.100347) (0.2,0.042957) (0.3,0.015106) (0.4,0.003017) (0.5,0) (0.6,0.001939) (0.7,0.006288) (0.8,0.011604) (0.9,0.017143) (1,0.022542)
    };
    \addlegendentry{lower bound}

    \addplot[color=darkGreen,
    solid,
    line width=1pt,
    mark=triangle,
    mark options={solid},
    mark size=2.5pt] coordinates {
        (0,0.25) (0.1,0.147781) (0.2,0.094201) (0.3,0.0625) (0.4,0.0625) (0.5,0.0625) (0.6,0.0625) (0.7,0.0625) (0.8,0.0625) (0.9,0.0625) (1,0.0625)
    };
    \addlegendentry{simple average}
    
    \addplot[color=darkBlue,
    solid,
    line width=1pt,
    mark=square,
    mark options={solid},
    mark size=2.5pt] coordinates {
        (0,0.249999) (0.1,0.112599) (0.2,0.054847) (0.3,0.02508) (0.4,0.009709) (0.5,0.002478) (0.6,0.002237) (0.7,0.006743) (0.8,0.013867) (0.9,0.022527) (1,0.032162)
    };
    \addlegendentry{our aggregator}

    \addplot[color=lightGreen,
    solid,
    line width=1pt,
    mark=diamond,
    mark options={solid},
    mark size=2.5pt] coordinates {
        (0,0.25) (0.1,0.147781) (0.2,0.094201) (0.3,0.060114) (0.4,0.037578) (0.5,0.022542) (0.6,0.015458) (0.7,0.014549) (0.8,0.014352) (0.9,0.018257) (1,0.025992)
    };
    \addlegendentry{average prior}
    \end{axis}
    \end{tikzpicture}
    \caption{Our aggregator vs. Existing aggregators}
    \label{fig: main results}
\end{figure}

\paragraph{\textbf{Surprising Benefits of Base Rate Neglect}} When we have a single expert, we prefer this expert to be a perfect Bayesian. The case becomes more complex with two experts. Intuitively, we might expect that having two perfect Bayesian experts would be best. However, our results suggest there might be unexpected advantages if experts neglect the base rate to some extent.

We show that the regret curve for any aggregator must be single-troughed regarding $\lambda$ (first decreasing and then increasing, or monotone).
By numerical methods, we find many aggregators can achieve lower regret when $\lambda<1$, thus having V-shaped regret (first decreasing and then increasing), including existing aggregators, for example, the average prior aggregator (see Figure~\ref{fig: main results}), that are particularly designed for perfect Bayesian \citep{arieli2018robust}. 

We analyze the optimal regret $\inf_g{R_\lambda(g)}$ for each $\lambda$ value. Due to the complexity of finding optimal aggregator, 
we provide tight lower bounds and numerical upper bounds for $\inf_g{R_\lambda(g)}$, with a small margin of error up to 0.003. We prove that the lower bound on worst-case regret is V-shaped as $\lambda$ increases (Theorem \ref{theorem: unimodality of the lower bound}). Moreover, the numerical upper bound is also V-shaped. 


Specifically, for $\lambda=0.5$, there exists an aggregator that can achieve almost-zero regret. However, \citet{arieli2018robust} validate that when experts are perfect Bayesian, no aggregator can have a regret less than 0.0225. In other words, when experts' prior consideration degree is $\lambda=0.5$, there exists an aggregator that outperforms all aggregators with perfect Bayesian posterior input.

The above counter-intuitive findings reveal the benefits of base rate neglect in aggregation. Here is an intuitive explanation. When experts make predictions, they use two main types of information: the shared information (the base rate) and the private information. An effective aggregator needs to balance these types in an appropriate proportion. However, an ignorant aggregator cannot correctly decompose these two kinds of information and may overemphasize the base rate in the aggregation because the base rate is repeatedly considered by the two experts. To address this, prior studies recommend using additional information, such as historical data and second-order information, to downplay the base rate's influence \citep{kim2001inefficiency, chen2004eliminating, palley2019extracting}.

In scenarios where experts lean towards disregarding the base rate, particularly when a parameter $\lambda$ is adjusted from $1$ to $0.5$, the issue of base rate double-counting diminishes. Thus, the aggregator has a chance to perform better.



\paragraph{\textbf{New Aggregators: Balancing the Base Rate}} We provide a closed-form aggregator $f$ with numerical regret $R(f)$ of only 0.013 (see our aggregator in Figure~\ref{fig: main results}). This demonstrates nearly optimal performance without knowing experts' true prior consideration degree $\lambda$. In detail, we design a family of $\hat{\lambda}$-base rate balancing aggregators. Each of them assumes the experts incorporate the prior at a specific $\hat{\lambda}$ degree and balance the commonly shared prior and experts' private insights under this assumption. These aggregators do not know the exact prior value. Instead, they use the average of experts' predictions as a proxy of the prior just as what an existing aggregator, the average prior, does. Particularly, the average prior aggregator is a special one of this family with $\hat{\lambda}$ assumed to be $1$. With $\hat{\lambda} = 0.7$, we get the aggregator shown in Figure~\ref{fig: main results} which performs generally well for all $\lambda$. 



\paragraph{\textbf{Empirical Evaluation of Aggregators}} To empirically quantify the consideration degree of base rate and evaluate the performance of various aggregators, we conduct a study to gather predictions across tens of thousands of discrete information structures spanning the entire spectrum. 
The results are multidimensional. 
First, people exhibit a significant degree of heterogeneity, with some ignoring the base rate ($\lambda$ approaching 0), and some applying the Bayesian rule ($\lambda$ approaching 1). A certain proportion of participants fall outside the theoretical range between perfect base rate neglect and Bayesian. For instance, some place very high emphasis on the base rate, or even report only the base rate itself. Furthermore, simple average aggregator outperforms the family of $\hat{\lambda}$-base rate balancing aggregators 
in terms of square relative loss in the whole sample. 
However, when focusing on the subset of predictions exhibiting base rate neglect, there are some $\hat{\lambda}$-base rate balancing aggregators ($\hat{\lambda} < 1$) that performs better than both simple average and average prior aggregators. 
Lastly, base rate neglect alone does not compromise aggregation performance as long as an appropriate $\hat{\lambda}$-base rate balancing aggregator is chosen.



\subsection{Related Work}
Forecast aggregation is widely studied. Many studies explore various aggregating methodologies theoretically and empirically such as \citet{clemen1986combining, stock2004combination, jose2008simple, baron2014two, satopaa2014combining}. 
Our work focuses on prior-free forecast aggregation, where an ignorant aggregator without access to the exact information structure is required to integrate predictions provided by multiple experts. There exists a body of work that studies the performance of the ignorant aggregator in a robust framework, where aggregators' efficacy is measured by the worst-case among a set of possible information structures. 


\paragraph{\textbf{Robust Aggregation}}
\citet{arieli2018robust} propose this robust framework by considering an additive regret formulation compared to an omniscient benchmark. In this study, low-regret aggregators for two agents are presented under the assumptions of Blackwell-ordered and conditionally independent structures. \citet{neyman2022you} consider aggregators with low approximation ratio under both the prior-free setting and a known prior setting where the aggregator knows not only the experts’ predictions but also the prior likelihood of the world state. Their analysis is performed within a set of informational substitutes structures, which is termed \sout{as} projective substitutes. \citet{levy2022combining} study the robust prediction aggregation under a setting where the marginal distributions of the forecasters are known but their joint correlation structure is unobservable.  
\citet{de2021robust} consider a similar setting to \citet{levy2022combining} while studying a robust action decision problem where an optimal action is selected among a finite action space based on multiple experiment realizations whose isolated distribution is known. In addition, \citet{babichenko2021learning} considers the forecast aggregation problem in a repeated setting, where the optimal forecast at each period is considered as the benchmark. \citet{guo2024algorithmic} propose an algorithmic framework for general information aggregation with a finite set of information structures.

All the above work assumes experts are Bayesian. In contrast, we consider the case where experts display base rate neglect. Such bias is widely studied in economic and psychological literature.

\paragraph{\textbf{Base Rate Neglect}} Start from seminal work of \citet{kahneman1973psychology}, a series of studies focus on the phenomenon of deviation from Bayesian updating by ignoring the unconditional probability, which is named base rate base rate neglect. The bias is examined across various subjects, including doctors \citep{eddy1982m}, law students \citep{eide2011two}, or even pigeons \citep{fantino2005teaching}. See the related survey papers for a systematic review of research related to base rate neglect \citep{koehler1996base,barbey2007base,benjamin2019errors}. 



Early studies mainly focus on the psychological mechanism explaining base rate neglect \citep{kahneman1973psychology,nisbett1976popular,bar1980base}. Then researchers begin to investigate the factors that may influence the degree of base rate neglect, such as uninformative description [e.g., \citealp{fischhoff1984diagnosticity,ginossar1987problem,gigerenzer1988presentation}], training and feedback \citep{goodie1999does,esponda2023mental}, framing \citep{barbey2007base}, variability of prior and likelihood information \citep{yang2020base}. For example, \citet{esponda2023mental} investigate the persistent base rate neglect when feedback is provided, and examine several potential mechanisms that inhibit the effect of learning. 

Recent works provide new mechanisms and implications to understand base rate neglect. For instance, \citet{yang2020base} further illustrate the neurocomputational substrates of base rate neglect. \citet{benjamin2019base} extend the previous formalizations of base-rate neglect and broadly examine its implications such as persuasion and reputation-building. However, few studies consider the impact of base rate neglect and how to deal with predictions based on it, especially in the process of information aggregation.

%% file: color_table.tex
\definecolor{lightBlue}{HTML}{a6cee3}
\definecolor{darkBlue}{HTML}{1f78b4}
\definecolor{lightGreen}{HTML}{b2df8a}
\definecolor{darkGreen}{HTML}{33a02c}
\definecolor{lightRed}{HTML}{fb9a99}

\definecolor{c1}{HTML}{ccebc5}
\definecolor{c2}{HTML}{a8ddb5}
\definecolor{c3}{HTML}{7bccc4}
\definecolor{c4}{HTML}{43a2ca}
\definecolor{c5}{HTML}{0868ac}

%% file: 2_problem_statement.tex
\section{Problem Statement}\label{section: problem statement}
We follow \citet{arieli2018robust}'s setting: There are two possible world states $\omega \in \Omega = \{0,1\}$. Two experts each receive a private signal that provides information about the current world state. For expert $i$, the signal $S_i$ comes from a discrete signal space $\mathcal{S}_i$. The overall signal space for all experts is denoted as $\mathcal{S} = \mathcal{S}_1 \times \mathcal{S}_2$. 

The relationship between the world states and the experts' signals is characterized by the information structure, $\theta$, which belongs to the set $\Delta_{\Omega \times \mathcal{S}}$. In this work, we assume the experts' signals are independent conditional on the world state. 
 We denote the set of information structures that align with this assumption as $\Theta$.

While experts are aware of the information structure $\theta$ and receive private signals, there is a decision maker who is uninformed about $\theta$ but interested in determining the true world state $\omega$. The decision maker obtains predictions from the experts regarding the likelihood of $\omega$ being 1. These predictions may vary as each expert has access to different signals. An aggregator is required to integrate experts' predictions into an aggregated forecast.

Formally, an aggregator is a deterministic function $f:[0,1]^2 \to [0,1]$, which maps experts' prediction profile $\mathbf{x} = (x_1, x_2)$ to a single aggregated result.  
The decision maker wants to find a robust aggregator that works well across all possible information structures in $\Theta$.

Unlike previous work by \citet{arieli2018robust} 
where the experts are modeled as Bayesian agents, we consider experts' base rate fallacy and employ the model introduced in the introduction. The relationship between the perfect Bayesian and the posterior that considers base rate neglect has been stated in the introduction. We defer the proof to Appendix~\ref{app:ps}.  


\begin{appendixproof} 
\label{app:ps}
The Bayesian posterior of expert $i$ upon receiving signal $s_i$ is \[\bayes_i(s_i)=\frac{\mu\cdot \Pr[S_i = s_i\mid \omega = 1]}{\mu\cdot \Pr[S_i = s_i\mid \omega = 1]+(1-\mu)\cdot \Pr[S_i = s_i\mid \omega = 0]}.\]
By normalizing the numerator of the Bayesian posterior, we simplify the expression to 
\[\bayes_i(s_i)=\cfrac{1}{1 + \cfrac{1-\mu}{\mu} \cdot \cfrac{\Pr[S_i = s_i\mid \omega = 0]}{\Pr[S_i = s_i\mid \omega = 1]}}.\]
Further transforming this expression, we get
\[\frac{1 - \bayes_i(s_i)}{\bayes_i(s_i)}=\frac{1}{\bayes_i(s_i)} - 1 = \frac{1-\mu}{\mu} \cdot \frac{\Pr[S_i = s_i\mid \omega = 0]}{\Pr[S_i = s_i\mid \omega = 1]}.\]
    
    Analogously, for the expert's prediction $\brn_i(s_i, \lambda)$, \[\frac{1-\brn_i(s_i, \lambda)}{\brn_i(s_i, \lambda)} = \left(\frac{1-\mu}{\mu}\right)^{\lambda} \cdot \frac{\Pr[S_i = s_i\mid \omega = 0]}{\Pr[S_i = s_i\mid \omega = 1]}.\]
Thus, \[\frac{1 - \brn_i(s_i, \lambda)}{\brn_i(s_i, \lambda)} = \left(\frac{1-\mu}{\mu}\right)^{\lambda - 1} \cdot\frac{1 - \bayes_i(s_i)}{\bayes_i(s_i)}.\]

Taking the logarithm of these ratios, we derive 
    \[ \logit (\brn_i(s_i, \lambda))=\logit(\bayes(s_i))-(1-\lambda)\logit(\mu).\]
    
Moreover, consider  $\frac{1}{p} - 1 = \frac{1 -p}{p}$ and view $\brn_i(s_i, \lambda)$ as $p$, we have \[\brn_i(s_i, \lambda) = \cfrac{1}{1 + \frac{(1-\mu)^{\lambda - 1}}{\mu^{\lambda - 1}} \cdot \frac{1 - \bayes_i(s_i)}{\bayes_i(s_i)}}.\]

Further transformation derives  \[
\brn_i(s_i, \lambda) = \frac{(1-\mu)^{1-\lambda}\cdot \bayes(s_i) }{(1-\mu)^{1-\lambda}\cdot \bayes(s_i)+\mu^{1-\lambda}\cdot (1-\bayes(s_i))}
.\]
\end{appendixproof}

As a preliminary step in the investigation of the base rate fallacy in information aggregation, we assume both experts have a consistent consideration degree of base rate.

\subsection{Aggregator Evaluation}\label{section: criteria}
To evaluate the performance of an aggregator $f$, we adopt the regret definition from \citet{arieli2018robust}. For a given base rate consideration degree $\lambda$, the regret of an aggregator $f$ is defined as:
$$R_\lambda(f) = \sup_{\theta \in \Theta}{\mathbb{E}_{\theta}
[ L(f(\mathbf{x}(\mathbf{s}, \lambda)), \omega) - L(f^*(\mathbf{s}), \omega)]}.$$

In this definition, an unachievable \textit{omniscient aggregator} $f^*$, who knows the information structure $\theta$ and all experts' signals and outputs the Bayesian posterior, serves as a benchmark. Let $f^*(\mathbf{s})$ denote the Bayesian posterior upon signal profile $\mathbf{s} = (s_1, s_2)$.
In contrast, the aggregator $f$ does not know $\theta$ and only inputs the experts' prediction profile $\mathbf{x}(\mathbf{s}, \lambda) = (\brn_1(s_1, \lambda), \brn_2(s_2,\lambda))$.

Formula $L(f(\mathbf{x}(\mathbf{s}, \lambda)), \omega) - L(f^*(\mathbf{s}), \omega)$ corresponds to the accuracy loss of aggregator $f$ compared to $f^*$ on signal profile $\mathbf{s}$ and true world state $\omega$, where we use loss function $L:[0,1] \times\Omega \to \mathbb{R}^+$ to measure the forecast accuracy. Particularly, we employ square loss, i.e., $L(p, \omega) = (p - \omega)^2$.
The relative loss of $f$ is computed as the expected accuracy loss, where the expectation is taken over the sampling of the truth state and signals. We also name this relative loss as the regret at some structure $\theta$ later.

The regret $R_\lambda(f)$ considers the worst-case relative loss, whereas the worst-case refers to the information structure that maximizes the relative loss. As mentioned in the introduction, we propose a new framework that measures the overall regret of aggregator $f$ under unknown prior consideration degree $\lambda$: $R(f) = \sup_{\lambda \in [0, 1]}\{R_\lambda(f) - \inf_{g}{R_\lambda(g)}\}.$
This definition quantifies the maximal gap between the regret of aggregator $f$ and the optimal regret achievable by the best possible aggregator $g$. An aggregator with a low overall regret performs well for every possible $\lambda$.


The below is a useful claim that we will repeatedly use with squared loss.  

\begin{claim}[Alternative Formula for the Relative Loss \cite{arieli2018robust}]\label{lemma: square relative loss}
The relative square loss between $f$ and the omniscient aggregator $f^*$ can be expressed as:

    \[\mathbb{E} [(f(\mathbf{x}) - \omega)^2 - (f^*(\mathbf{s}) - \omega)^2] = \mathbb{E} [(f(\mathbf{x}) - f^*(\mathbf{s}))^2]\]
\end{claim}


The relative loss can be written as the expected squared loss between $f$ and $f^*$ under the square loss function. We defer the proof of this claim to Appendix~\ref{app:claim loss}. Intuitively, the closer the aggregated forecast $f(\mathbf{x})$ is to the omniscient prediction $f^*(\mathbf{s})$, the smaller the relative loss becomes. If an aggregator can output the Bayesian aggregator's posterior at some structure $\theta$, then the relative loss of it under this $\theta$ is exactly zero.

\begin{appendixproof}
\label{app:claim loss}
    We prove this equation for any signal profile $(s_1, s_2)$ and any report profile $(x_1, x_2)$.
   \begin{align*}
&\mathbb{E}\left[\left(f(x_1, x_2)- \omega\right)^2 - \left(f^*(s_1, s_2)-\omega\right)^2\mid S_1 = s_1, S_2 = s_2\right] \\
&\mathbb{E}\left[f(x_1, x_2)^2 - 2\omega\cdot f(x_1, x_2) - f^*(s_1, s_2)^2 +2\omega\cdot f^*(s_1, s_2) \mid S_1 = s_1, S_2 = s_2 \right] \\
= &f(x_1,x_2)^2 - f^*(s_1,s_2)^2 + 2 \mathbb{E}\left[\omega \mid S_1 = s_1, S_2 = s_2 \right]\cdot \left(f^*(s_1, s_2) - f(x_1,x_2)\right) \\
      = & f(x_1, x_2)^2 - f^*(s_1, s_2)^2  + 2f^*(s_1, s_2)\left(f^*(s_1, s_2) - f(x_1,x_2)\right)\\
      = &f(x_1, x_2)^2 + f^*(s_1,s_2)^2  - 2f^*(s_1, s_2)\cdot  f(x_1,x_2)\\
      = &\left(f(x_1, x_2)- f^*(s_1, s_2)\right)^2
   \end{align*}
\end{appendixproof}



%% file: 3_omniscient_aggregator.tex
\section{Warm Up: the Omniscient Aggregator}
As we mentioned before, the omniscient aggregator $f^*$ is compared to assess the aggregator's regret.
This omniscient aggregator possesses complete knowledge about the underlying information structure $\theta$ and experts' signals. It works as a Bayesian aggregator that takes experts' signals as input and utilizes its knowledge about $\theta$ to output the Bayesian posterior upon experts' signals. 
Formally, \begin{align*}
f^*(s_1, s_2) = &\Pr[\omega = 1 \mid S_1 = s_1, S_2 = s_2]\\
= &\cfrac{\Pr[\omega = 1, S_1 = s_1, S_2 = s_2]}{\sum_{\sigma \in \{0,1\}}\Pr[\omega = \sigma, S_1 = s_1, S_2 = s_2]}.
\end{align*}
Particularly, in our conditionally independent setting, the calculation of this Bayesian aggregator's posterior does not rely on the knowledge of joint distribution $\theta$. The experts' predictions, the base rate consideration degree $\lambda$, and the prior $\mu$ are enough to obtain this posterior.

\begin{observation}\label{lemma: omniscient aggregator}
    Given the prior $\mu$, the base rate consideration degree $\lambda$, and the experts' prediction profile $(x_1, x_2) = (\brn_1(s_1, \lambda), \brn_2(s_2,\lambda))$, the Bayesian aggregator’s posterior is \[f^*(s_1, s_2) = \frac{(1-\mu)^{2\lambda - 1}x_1x_2}{(1-\mu)^{2\lambda - 1}x_1x_2 + \mu^{2\lambda - 1}(1-x_1)(1-x_2)}.\]
\end{observation}
We defer the proof to Appendix~\ref{app:omniscient}.
The conditionally independent assumption plays a crucial role in formulating the aggregator's posterior through the individual predictions of experts. 

When $\lambda = 0$,the prediction profile $(x_1, x_2) = (\brn_1(s_1, 0), \brn_2(s_2,0))$ showcases the relative ratio in frequencies of signals $(s_1, s_2)$ under state $\omega = 1$ compared to their frequencies under state $\omega = 0$. 
The aggregated result in this case is given by \(\frac{\mu x_1x_2}{\mu x_1x_2 + (1-\mu)(1-x_1)(1-x_2)}.\) 
As $\lambda$ increases, indicating a higher degree of prior consideration by the experts, the influence of $\mu$ in the Bayesian aggregator's posterior is correspondingly diminished.
When $\lambda = 1$,  profile $(x_1, x_2) = (\brn_1(s_1, 1), \brn_2(s_2,1))$ corresponds to individual experts' Bayesian posteriors. The aggregation formula becomes \(\frac{(1-\mu)x_1x_2}{(1-\mu)x_1x_2 + \mu(1-x_1)(1-x_2)}\).
    

\begin{appendixproof}[Proof of Observation~\ref{lemma: omniscient aggregator}]
\label{app:omniscient}
For a concise presentation, we shorten the notation $\bayes_i(s_i)$ and $\brn_i(s_i, \lambda)$ to $\bayes_i$ and $\brn_i$ in this proof. 

In our conditionally independent setting, the Bayesian posterior upon signal profile $\mathbf{s} = (s_1, s_2)$ can be rewritten using the prior distribution of the state and the perfect Bayesian posteriors of the experts as below. 
\begin{align*}
    \Pr[\omega = 1 \mid S_1 = s_1, S_2 = s_2] = &\frac{\Pr[\omega = 1, S_1 = s_1, S_2 = s_2]}{\sum_{\sigma \in \{0,1\}}\Pr[\omega = \sigma, S_1 = s_1, S_2 = s_2]}\\
= &\cfrac{\Pr[\omega = 1] \prod_{i \in \{1,2\}}{\Pr[S_i = s_i \mid \omega = 1]}}{\sum_{\sigma \in \{0,1\}}\Pr[\omega = \sigma] \prod_{i \in \{1,2\}}{\Pr[S_i = s_i \mid \omega = \sigma]}} \tag{by conditionally independent assumption}\\
= &\frac{\Pr[\omega = 1]^{-1} \prod_{i \in \{1,2\}}{\left(\Pr[S_i = s_i]\cdot \Pr[\omega = 1\mid S_i = s_i]\right)}}{\sum_{\sigma \in \{0,1\}}{\Pr[\omega = \sigma]^{-1} \prod_{i \in \{1,2\}}{\left(\Pr[S_i = s_i]\cdot \Pr[\omega = \sigma\mid S_i = s_i]\right)}}}\tag{by Bayes' Theorem}\\
&= \frac{\mu^{-1}\bayes_1\bayes_2}{\mu^{-1}\bayes_1\bayes_2 + (1 - \mu)^{-1}(1 - \bayes_1)(1 - \bayes_2)}.\tag{by the law of total probability, $\Pr[\omega = 0\mid S_i = s_i] = 1 - \bayes_i$}
\end{align*}

Using Observation~\ref{lemma: BNR prediction vs Bayesian}, we can replace the perfect Bayesian posterior $\bayes_i$ by the expert's prediction which exhibits base rate neglect $\brn_i$ and the prior consideration degree $\lambda$:
\begin{align*}
    \Pr[\omega = 1 \mid S_1 = s_1, S_2 = s_2] & = \cfrac{1}{1 + \cfrac{\mu}{(1 - \mu)}\cdot\cfrac{1 - \bayes_1}{\bayes_1}\cdot\cfrac{1 - \bayes_2}{\bayes_2}}\\
    & = \cfrac{1}{1 + \cfrac{\mu}{(1 - \mu)}\cdot\left(\cfrac{1-\mu}{\mu}\right)^{(1-\lambda)\times 2}\cdot\cfrac{1 - \brn_1(s_1, \lambda)}{\brn_1(s_1, \lambda)}\cdot\cfrac{1 - \brn_2(s_2, \lambda)}{\brn_2(s_2, \lambda)}}\\
    & = \cfrac{1}{1 + \cfrac{\mu^{2\lambda - 1}}{(1 - \mu)^{2\lambda - 1}}\cdot\cfrac{1 - \brn_1(s_1, \lambda)}{\brn_1(s_1, \lambda)}\cdot\cfrac{1 - \brn_2(s_2, \lambda)}{\brn_2(s_2, \lambda)}}\\
    &= \cfrac{(1-\mu)^{2\lambda - 1}\brn_1\brn_2}{(1-\mu)^{2\lambda - 1}\brn_1\brn_2 + \mu^{2\lambda - 1}(1-\brn_1)(1-\brn_2)}.
\end{align*}

\end{appendixproof}

%% file: 4_Regret_Curve_Theory.tex

In this section, we study how the degree $\lambda$ affects regret. Our theoretical results demonstrate the single-trough of all regret curves. 




\begin{theorem}[Regret Curves Are Single-troughed]\label{theorem: V-shaped upper bound}
    For any aggregator $f: [0,1]^2 \to [0,1]$, the regret $R_\lambda(f)$ is either monotone or first monotonically decreasing and then monotonically increasing for the base rate consideration degree $\lambda$. 
    We call such curves single-troughed. 
\end{theorem}

According to our definition, monotone functions are also single-troughed. Thus, we additionally define non-monotone single-troughed functions as V-shaped functions to distinguish. Intuitively, as the degree $\lambda$ increases, the experts become more Bayesian, and the aggregator's regret may decrease. However, Section~\ref{section: numerical results} illustrates the non-monotonicity, and thus, the V-shape of many aggregators, including the average prior aggregator which was previously designed to aggregate Bayesian experts~\cite{arieli2018robust}.


The key observation used in proving this theorem is that the supremum of a family of single-troughed functions is still single-troughed. Though there does not exist a closed-form format for $R_\lambda(f)$, we will prove that $R_\lambda(f)$ is the supremum of a family of ``simple'' single-troughed functions. 

To achieve this, we first reduce the regret computation to a smaller structure space
, where each expert only receives two types of signals, i.e., signal $r$ or signal $b$ (we denote this space $\Theta_4$ because there are four distinct signals in total).
Then we construct a family of transformations on\hide{ the two-signal structure space} $\Theta_4$, denoted as $\{t_\theta\}_{\theta \in \Theta_4}$, and build a family of relative loss functions, each being single-troughed, denoted as $\{\phi_\theta\}_{\theta\in\Theta_4}$. 
In a transformation $t_\theta$, structure $\theta$ is adapted according to experts' prior consideration degree $\lambda$. At each value of $\lambda$, the adapted structure $t_\theta(\lambda)$ induces the same expert predictions as the perfect Bayesian posteriors under structure $\theta$ and then generates a specific relative loss. For a fixed structure $\theta$, its adaptations across different $\lambda$ values (i.e., $\{t_\theta(\lambda)\}_{0 < \lambda \le 1}$) derive the relative loss curve $\phi_\theta$. Moreover, if we fix the prior consideration degree $0 <\lambda \le 1$, then the ensemble of adapted structures at $\lambda$ degree (i.e. $\{t_\theta(\lambda)\}_{\theta\in\Theta_4}$) make up the whole structure space $\Theta_4$. Therefore, the regret function $R_\lambda(f)$, which assesses the supremum loss across all information structures at each point, can be viewed as the supremum of loss functions.





\begin{proof}[Proof of Theorem~\ref{theorem: V-shaped upper bound}]

To analyze the property of $R_\lambda(f)$, we first reduce the regret calculation of $f$ from the loss supremum across all conditionally independent structures in $\Theta$ to the loss supremum across two-signal structures in $\Theta_4$, i.e., $|\mathcal{S}_i| = 2$ for $i = 1,2$. We formally describe the statement as below. 

\begin{lemma}[Reduction of Regret Computation]\label{lemma: reduction}
    In the conditionally independent setting, for any aggregator $f$ and any base rate consideration degree $\lambda$, $$R_\lambda(f) = \sup_{\theta \in \Theta_4}{\mathbb{E}_{\theta}
[ L(f(\mathbf{x}(\mathbf{s}, \lambda)), \omega) - L(f^*(\mathbf{s}), \omega)]},$$ where $\Theta_4$ is the set of all two-signal structures.
\end{lemma}

\input{Theory_V-shape_Regret/Theory_Regret_Reduction}


By this lemma, we can only consider two-signal information structures in $\Theta_4$ in the following proof. For simplicity, we denote this kind of structures as quintuples $(\mu, \alpha_1, \beta_1, \alpha_2, \beta_2)$, where the experts' signal space is $\mathcal{S}_1 \times \mathcal{S}_2 = \{r, b\}^2$, parameter $\mu$ corresponds to the prior probability of state $\omega = 1$, and parameters $\alpha_i$, $\beta_i$ are the conditional probabilities of receiving signal $r$ given the world state $\omega = 1$ or $\omega = 0$ for expert $i$, 
i.e., $\Pr[S_i = r \mid \omega = 1] = \alpha_i$ and $\Pr[S_i = r \mid \omega = 0] = \beta_i$. 

The key to our proof lies in transforming the regret function, which is defined by the supremum loss at each point, 
into the supremum of a set of loss functions. This is achieved by introducing a family of transformations. 
A transformation \( t_\theta \) is a mapping from $(0,1]$ to $\Theta_4$, which adapts $\theta$ according to the prior consideration degree \( \lambda \). Formally, for structure \( \theta = (\mu, \alpha_1, \beta_1, \alpha_2, \beta_2)\), we define the adapted structure \( t_\theta(\lambda)\) as \((u(\lambda), \alpha_1, \beta_1, \alpha_2, \beta_2) \), where 
\[ u(\lambda) = \left(1 + \left(\frac{1 - \mu}{\mu}\right)^{1/\lambda}\right)^{-1}. \]
By this definition, the equation \(\left(\cfrac{u(\lambda)}{1-u(\lambda)}\right)^\lambda = \cfrac{\mu}{1 - \mu}\) holds.
This ensures that the expert's prediction \( \brn_i(s_i, \lambda) \) in adapted structure \( t_\theta(\lambda)\), denoted as $x_i^{s_i}(\lambda)$, 
is the same as Bayesian posterior $\bayes_i(s_i)$ in structure \( \theta \). In other words, the expert's prediction upon the same signal $s_i$ is constant when the structure varies with the expert's prior consideration degree $\lambda$ in the rule of $t_\theta$. We denote this constant prediction value as $x_i^{s_i}$, which is exactly the Bayesian posterior $\bayes_i(s_i)$ in structure $\theta$. 

Each transformation $t_\theta$ induces a loss curve $\phi_\theta$, where the value at $\lambda$ 
is the relative loss of $f$ in structure $t_\theta(\lambda)$. 
That is, $$\phi_\theta(\lambda) = \mathbb{E}_{t_\theta(\lambda)}
[ L(f(\mathbf{x}(\mathbf{s}, \lambda)), \omega) - L(f^*(\mathbf{s}), \omega)].$$ The loss function $\phi_\theta$ is simple, with a definitive closed form. Through calculus derivation analyses, we can demonstrate the single-trough. 
\begin{lemma}[Relative Loss Is Single-troughed]\label{lemma: V-shaped relative loss}
For each $\theta \in \Theta_4$, the regret curve \( \phi_\theta(\lambda) \) is single-troughed for \( \lambda \).
\end{lemma}
\input{Theory_V-shape_Regret/Theory_V-shape_Relative_Loss}


The loss curve $\phi_\theta$ is obtained by fixing structure $\theta$ and varying degree $\lambda$. When we instead fix degree $\lambda$ and vary the anchor structure $\theta$, we find that $t_\theta(\lambda)$ covers the structure space. Formally, \[ \left\{t_\theta(\lambda) \mid \theta \in \Theta_4\right\} =\Theta_4,~\forall \lambda \in (0,1]. \]
Therefore, regret $R_\lambda(f)$ can be expressed as supremum of loss, 
\[ R_\lambda(f) = \sup_{\theta \in \Theta_4}\left\{ \phi_\theta(\lambda)\right\},~\forall \lambda \in (0,1]. \]

Applying the following lemma, which verifies that the supremum operation preserves the single-trough, 
we can conclude that \( R_\lambda(f) \) is single-troughed for degree  \( \lambda \).

\begin{lemma}[Supremum of Single-troughed Functions Is Single-troughed]\label{lemma: V-shaped supremum}
    Let $\{f_\alpha\}_{\alpha \in I}$ be a series of single-troughed functions defined on the interval $[a, b]$. Let $\hat{f}$ be the supremum of these functions, i.e., $$\hat{f}(x) = \sup_{\alpha \in I}\{f_\alpha(x)\}\text{~for all~} x \in [a,b].$$ Then, $\hat{f}$ is single-troughed in $[a,b]$.
\end{lemma}
\input{Theory_V-shape_Regret/Theory_V-shape_supremum}
The proofs of Lemma~\ref{lemma: reduction}, Lemma~\ref{lemma: V-shaped relative loss}, and Lemma~\ref{lemma: V-shaped supremum} are deferred to Appendix~\ref{app:v shape}.

\end{proof}

%% file: Theory_V-shape_Regret/Theory_Regret_Reduction.tex
\begin{appendixproof}[Proof of Lemma~\ref{lemma: reduction}]
\label{app:v shape}
    The inequality $$\sup_{\theta \in \Theta_4}{\mathbb{E}_{\theta}
[ L(f(\mathbf{x}(\mathbf{s}, \lambda)), \omega) - L(f^*(\mathbf{s}), \omega)]} \le \sup_{\theta \in \Theta}{\mathbb{E}_{\theta}
[ L(f(\mathbf{x}(\mathbf{s}, \lambda)), \omega) - L(f^*(\mathbf{s}), \omega)]}$$ follows directly by the fact that $\Theta_4 \subseteq \Theta$.

We're next to show $$\sup_{\theta \in \Theta_4}{\mathbb{E}_{\theta}
[ L(f(\mathbf{x}(\mathbf{s}, \lambda)), \omega) - L(f^*(\mathbf{s}), \omega)]} \ge \sup_{\theta \in \Theta}{\mathbb{E}_{\theta}
[ L(f(\mathbf{x}(\mathbf{s}, \lambda)), \omega) - L(f^*(\mathbf{s}), \omega)]}$$ by decomposing each $\theta \in \Theta$ and getting a ``basic'' structure $\theta'$ in $\Theta_4$ with higher regret.


Let $b_i^s$ denote the Bayesian posterior of expert $i$ upon receiving signal $s$, i.e., $b_i^s = \bayes_i(s)$, $q_i^s$ denote the prior probability that expert $i$ receives signal $s$, i.e., $q_i^s = \Pr[S_i = s]$.
Let $\mathbf{b}_i = \left(b_i^s\right)_{s \in \mathcal{S}_i}$ and $\mathbf{q}_i = \left(q_i^s\right)_{s \in \mathcal{S}_i}$ be the Bayesian posterior vector and the prior vector of expert $i$.

We perform the decomposition in a restricted space.
For a fixed structure $\theta$, we consider the structures that share the same prior, $\mu$, and the same Bayesian posterior vector $\mathbf{b}_i$. As is shown in the following claim, the regret of these structures exhibits multi-linear property at a fixed value of $\lambda$. 

\begin{claim}\label{claim: multilinear regret}
Fixing $\mu$, $\lambda$ and $\mathbf{b}_1$, $\mathbf{b}_2$, the regret is a multi-linear function of $\mathbf{q}_1$, $\mathbf{q}_2$. Formally, there exists a function $\psi$ such that the regret is $\mathbf{q}_1^\top\mathbf{\Psi}\mathbf{q_2}$ where $\mathbf{\Psi}_{s_1,s_2} = \psi(\mu, \lambda, b_1^{s_1}, b_2^{s_2}), \forall s_1 \in \mathcal{S}_1, s_2\in \mathcal{S}_2$.
\end{claim}

In addition, this restricted space we considered imposes some restrictions on the prior vectors $\mathbf{q}_i, i = 1,2$, which can be translated into linear constraints.

\begin{claim}\label{claim: LP of prior vector}
    For all $\theta \in \Theta$ with prior $\mu$ and Bayesian posterior vector $\mathbf{b}_1$, $\mathbf{b}_2$,  the prior vector $\mathbf{q}_i, i = 1,2$, satisfies the following linear constraints:
\[
\left\{
\begin{aligned}
    &(1) \sum_{s \in \mathcal{S}_i} q_i^s = 1,\\
    &(2) \sum_{s \in \mathcal{S}_i} q_i^s \cdot b_i^s = \mu,\\
    &(3) \forall s \in \mathcal{S}_i, q_i^s \ge 0.
\end{aligned}
\right.
\]

Moreover, any pair of vectors $\mathbf{q}_1$, $\mathbf{q}_2$ satisfying the above constraints, is prior vectors of some structure $\theta$ with prior $\mu$ and Bayesian posterior vector $\mathbf{b}_1$, $\mathbf{b}_2$.
\end{claim}

We defer the proofs of Claim~\ref{claim: multilinear regret} and Claim~\ref{claim: LP of prior vector} later.

By Claim~\ref{claim: LP of prior vector}, the prior vectors of $\theta$, we specifically mark as $\mathbf{q}_1(\theta)$ and $\mathbf{q}_2(\theta)$, are a solution of the linear programming. 

Then we decompose this solution.
According to the property of linear programming problem\cite{david1973introduction}, any solution $\mathbf{q}_i, i = 1,2$ of the linear constraints can be viewed as a convex combination of basic feasible solutions which have $\le 2$ non-zero entries.
We name these basic feasible solutions as basic prior vectors. 

According to the multilinear property of regret, there exists a pair of basic prior vectors $\mathbf{b}_i(\theta), i = 1,2$ such that the regret of $\mathbf{b}_1(\theta), \mathbf{b}_2(\theta)$ is at least the same as $\mathbf{q}_1(\theta)^\top\mathbf{\Psi}\mathbf{q_2}(\theta)$, which is the regret of $\theta$.

Using Claim~\ref{claim: LP of prior vector} again, we can construct the structure $\theta'$ in the restricted space, whose prior vectors are $\mathbf{b}_1(\theta), \mathbf{b}_2(\theta)$. The regret of $\theta'$ is $\mathbf{b}_1(\theta)^\top\mathbf{\Psi}\mathbf{b}_2(\theta)$, greater than or equal to $\mathbf{q}_1(\theta)^\top\mathbf{\Psi}\mathbf{q_2}(\theta)$, the regret of structure $\theta$. Only signals corresponding to the non-zero entries of $\mathbf{b}_i$ will be received by expert $i$ with a non-zero probability. Thus, the constructed structure $\theta'$ is a two-signal structure in $\Theta_4$ and we finish our proof.
\end{appendixproof}

\begin{appendixproof}[Proof of Claim~\ref{claim: multilinear regret}]
    The regret regarding $\theta$ is 
\begin{align*}
&\mathbb{E}_{\theta}[ L(f(\mathbf{x}(\mathbf{s}, \lambda)), \omega) - L(f^*(\mathbf{s}), \omega)]\\
=& \mathbb{E}_\theta [(f(\mathbf{x}(\mathbf{s}, \lambda)) - f^*(\mathbf{s}))^2]\tag{by Claim~\ref{lemma: square relative loss}}\\
=& \sum_{s_1, s_2} \Pr[S_1 = s_1, S_2 = s_2] \cdot (f(\brn_1(s_1, \lambda), \brn_2(s_2, \lambda)) - f^*(s_1, s_2))^2 \tag{expand the expectation}\\
=& \sum_{s_1, s_2} \Pr[S_1 = s_1, S_2 = s_2] \cdot (f(\phi(b_1^{s_1}), \phi(b_2^{s_2})) - f^*(s_1, s_2))^2 \tag{by Observation~\ref{lemma: BNR prediction vs Bayesian}, fixing $\mu$ and $\lambda$, $\brn_i(s_i, \lambda)$ is a function of Bayesian posterior $b_i^{s_i}$}\\
=& \sum_{s_1, s_2}  \Pr[S_1 = s_1, S_2 = s_2] \cdot \left( f(\phi(b_1^{s_1}), \phi(b_2^{s_2}))- \frac{(1-\mu)b_1^{s_1}b_2^{s_2}}{(1-\mu)b_1^{s_1}b_2^{s_2} + \mu(1-b_1^{s_1})(1-b_2^{s_2})}\right)^2\tag{by Observation~\ref{lemma: omniscient aggregator}}
\end{align*}
Moreover, the prior probability of signal profile $(s_1, s_2)$ is \begin{align*}
   \Pr[S_1 = s_1, S_2 = s_2] &= \sum_{\sigma \in \{0,1\}} \Pr[\omega = \sigma] \prod_{i \in \{1,2\}}\Pr[s_i \mid \omega = \sigma] \tag{by conditional independent assumption}\\
   &=\sum_{\sigma \in \{0,1\}} {\Pr[\omega = \sigma]\prod_{i \in \{1,2\}}\cfrac{\Pr[\omega = \sigma\mid s_i]\Pr[s_i]}{\Pr[\omega = \sigma]}} \tag{by Bayes' Theorem}\\
   &= \mu\cdot \cfrac{b_1^{s_1}q_1^{s_1}}{\mu}\cdot \cfrac{b_2^{s_2}q_2^{s_2}}{\mu} + (1-\mu)\cdot \cfrac{(1 - b_1^{s_1})q_1^{s_1}}{1-\mu}\cdot \cfrac{(1-b_2^{s_2})q_2^{s_2}}{1-\mu}.
\end{align*}
Thus, we have $\mathbb{E}_{\theta}[ L(f(\mathbf{x}(\mathbf{s}, \lambda)), \omega) - L(f^*(\mathbf{s}), \omega)] = \sum_{s_1, s_2}q_1^{s_1}q_2^{s_2}\psi(\mu, \lambda, b_1^{s_1}, b_2^{s_2})$.
\end{appendixproof}

\begin{appendixproof}[Proof of Claim~\ref{claim: LP of prior vector}]

For any $\theta \in \Theta$ with prior $\mu$ and Bayesian posterior vectors $\mathbf{b}_1, \mathbf{b}_2$, Constraint (1) and Constraint (3) are naturally satisfied since $\mathbf{q}_i, i=1,2$ is a probability distribution. Constraint (2) is satisfied because $$\mu = \Pr[\omega = 1] = \sum_{s \in \mathcal{S}_i} \Pr[\omega = 1, S_i = s] = \sum_{s \in \mathcal{S}_i} \Pr[S_i = s] \Pr[\omega = 1\mid S_i = s] =  \sum_{s \in \mathcal{S}_i} q_i^s \cdot b_i^s.$$

For the other direction, we will construct a structure $\theta \in \Theta$ for any solution $\mathbf{q}_i, i = 1, 2$ of the linear programming. 

Formally, for any pair of signals $s_1 \in \mathcal{S}_1, s_2 \in \mathcal{S}_2$, the joint distribution of signals and the world state is designed as
\begin{align*}
    \Pr[S_1 = s_1, S_2 = s_2, \omega = 1] &= \mu \cdot \cfrac{q_1^{s_1}b_1^{s_1}}{\mu}\cdot \cfrac{q_2^{s_2}b_2^{s_2}}{\mu}, \\
    \Pr[S_1 = s_1, S_2 = s_2, \omega = 0] &= 1 - \mu \cdot \cfrac{q_1^{s_1}(1 - b_1^{s_1})}{1 - \mu}\cdot \cfrac{q_2^{s_2}(1 - b_2^{s_2})}{1 - \mu}.
\end{align*}

By Constraint (1) and (2), we have  $$ \Pr[\omega = 1] = \sum_{s_1 \in \mathcal{S}_1, s_2 \in \mathcal{S}_2} \Pr[S_1 = s_1, S_2 = s_2, \omega = 1] =\mu^{-1} \left(\sum_{s_1 \in \mathcal{S}_1}q_1^{s_1}b_1^{s_1}\right)\left(\sum_{s_2 \in \mathcal{S}_2}q_2^{s_2}b_2^{s_2}\right) = \mu,$$ and $$\Pr[\omega = 0] = \sum_{s_1 \in \mathcal{S}_1, s_2 \in \mathcal{S}_2} \Pr[S_1 = s_1, S_2 = s_2, \omega = 0] = (1-\mu)^{-1} \left(\sum_{s_1 \in \mathcal{S}_1}q_1^{s_1}(1b_1^{s_1})\right)\left(\sum_{s_2 \in \mathcal{S}_2}q_2^{s_2}(1-b_2^{s_2})\right)  = 1-\mu.$$

Thus, $\sum_{s_1 \in \mathcal{S}_1, s_2 \in \mathcal{S}_2, \sigma \in \{0,1\}} \Pr[S_1 = s_1, S_2 = s_2, \omega = \sigma] = 1$, implying that the constructed structure $\theta$ is a valid joint distribution.

Moreover, for any signal $s_1 \in \mathcal{S}_1$, the Bayesian posterior upon receiving $s_1$ is \begin{align*}
    \Pr[\omega = 1 \mid S_1 = s_1] & = \cfrac{\sum_{s_2 \in \mathcal{S}_2}\Pr[S_1 = s_1, S_2 = s_2, \omega = 1]}{\sum_{\sigma \in \{0,1\}} \sum_{s_2 \in \mathcal{S}_2}\Pr[S_1 = s_1, S_2 = s_2, \omega = \sigma]}\\
    &=\cfrac{\mu^{-1}q_1^{s_1}b_1^{s_1} \sum_{s_2 \in \mathcal{S}_2}q_2^{s_2}x_2^{s_2}}{\mu^{-1}q_1^{s_1}b_1^{s_1} \sum_{s_2 \in \mathcal{S}_2}q_2^{s_2}b_2^{s_2} + (1-\mu)^{-1}q_1^{s_1}(1-b_1^{s_1}) \sum_{s_2 \in \mathcal{S}_2}q_2^{s_2}(1-b_2^{s_2})}\\
    &= \cfrac{q_1^{s_1}b_1^{s_1}}{q_1^{s_1}b_1^{s_1}  + q_1^{s_1}(1-b_1^{s_1})}\tag{by Constraint (1) and (2)}\\
    & = b_1^{s_1}.
\end{align*} 
The prior of signal $s_1$ is
\begin{align*}
    \Pr[S_1 = s_1] = &\sum_{s_2 \in \mathcal{S}_2, \sigma \in \{0,1\}}\Pr[S_1 = s_1, S_2 = s_2, \omega = \sigma] \\
    =& \mu^{-1}q_1^{s_1}b^{s_1}\sum_{s_2 \in \mathcal{S}_2}q_2^{s_2}b_2^{s_2} + (1-\mu)^{-1}q_1^{s_1}(1-b^{s_1})\sum_{s_2 \in \mathcal{S}_2}q_2^{s_2}(1-b_2^{s_2})\\
    =&q_1^{s_1}b^{s_1} + q_1^{s_1}(1-b^{s_1}) \tag{by Constraint (1) and (2)}\\
    =& q_1^{s_1}.
\end{align*}

Analogously, the Bayesian posterior upon receiving $s_2$ is $\Pr[\omega = 1 \mid S_2 = s_2] = b_i^{s_2}$, and the prior of $s_2$ is $\Pr[S_2 = s_2] = q_2^{s_2}$.

The above arguments demonstrate that the constructed $\theta$ meets all the conditions shown in our claim.
\end{appendixproof}

%% file: Theory_V-shape_Regret/Theory_V-shape_Relative_Loss.tex
\begin{appendixproof}[Proof of Lemma~\ref{lemma: V-shaped relative loss}]

For a given \( \theta \), the relative loss \( \phi_\theta(\lambda) \) can be broken down into contributions from all possible signal profiles \( \mathbf{s} \in \mathcal{S}\):
\[ \phi_\theta(\lambda) = \sum_{\mathbf{s} \in \mathcal{S}}\phi_\theta^\mathbf{s}(\lambda). \]

For a specific signal profile, for example, \( \mathbf{s} = (r,r) \), the relative loss for this profile can be expressed as a function of \( u(\lambda), \alpha_1, \alpha_2, \beta_1, \beta_2 \):
\begin{align*}
    \phi_\theta^{(r,r)}(\lambda) &= \Pr_{t_\theta(\lambda)}[S_1 = r,S_2 = r] \cdot \left[ L(f(x_1^r, x_2^r), \omega) - L(f^*(r,r), \omega)\mid t_\theta(\lambda)\right]\tag{as mentioned before, the expert's prediction $x_i^{s_i}(\lambda)$ is constant as $\lambda$ varies}\\
    &= \Pr_{t_\theta(\lambda)}[S_1 =r, S_2 =r] \cdot \left[ \left(f(x_1^r, x_2^r)- f^*(r,r) \right)^2 \mid t_\theta(\lambda)\right] \tag{by Claim~\ref{lemma: square relative loss}}\\
    &= \left[u(\lambda)\alpha_1\alpha_2 + (1-u(\lambda))\beta_1\beta_2\right]\cdot \left[f(x_1^{r}, x_2^{r}) - \cfrac{u(\lambda)\alpha_1\alpha_2}{u(\lambda)\alpha_1\alpha_2 + (1-u(\lambda))\beta_1\beta_2}\right]^2.
\end{align*}
 Here, \( f(x_1^{r}, x_2^{r}) \) represents the aggregation result when both the experts receive signal \(r\). In the rule of \( t_\theta \), this aggregation result is invariant as \( \lambda \) varies since the experts' report profile \((x_1^{s_1}(\lambda), x_2^{s_2}(\lambda))\) under the information structure \( t_\theta(\lambda) \) remains unchanged across different \( \lambda \) values.
 
Let $y_{(r,r)}(u)$ be $\cfrac{u\alpha_1\alpha_2}{u\alpha_1\alpha_2 + (1-u)\beta_1\beta_2}$, the omniscient aggregation result for signal profile $(r,r)$ under structure \((u, \alpha_1, \beta_1, \alpha_2, \beta_2)\).
The relative loss \(\phi_\theta^{(r,r)}(\lambda) \) can be simplified to 

\[ \phi_\theta^{(r,r)}(\lambda) = \alpha_1\alpha_2\cdot u(\lambda)\cdot \left[\frac{f(x_1^{r}, x_2^{r})^2}{y_{(r,r)}\left(u(\lambda)\right)} + y_{(r,r)}\left(u(\lambda)\right) - 2f(x_1^{r}, x_2^{r})\right]. \]
This form can be generalized for all signal profiles \( \mathbf{s} \). Namely, for all possible signal profile $\mathbf{s} \in \mathcal{S}$, $$\phi_\theta^{\mathbf{s}}(\lambda) = A\cdot u(\lambda)\cdot \left(\cfrac{C^2}{y(u(\lambda))} + y(u(\lambda)) -2C\right) \text{~where~}y(u) = \cfrac{1}{1 + K\frac{1-u}{u}}.$$
$A, C, K$ are constant for each $\mathbf{s}$, their values are listed in Table~\ref{tab: constant table}.

\setcounter{table}{0}
\renewcommand{\thetable}{D\arabic{table}}

\begin{table}[h]
    \centering
    \begin{tabular}{cccc}
    \hline\hline
         $\mathbf{s}$ & $A$ & $C$ & $K$ \\
    \hline
         $(r, r)$ & $\alpha_1\alpha_2$ & $f(x_1^{r}, x_2^{r})$ & $\beta_1\beta_2/\alpha_1\alpha_2$\\
         $(r, b)$ & $\alpha_1(1-\alpha_2)$ & $f(x_1^{r}, x_2^{b})$ & $\beta_1(1-\beta_2)/\alpha_1(1-\alpha_2)$\\
         $(b, r)$ & $(1-\alpha_1)\alpha_2$ & $f(x_1^{b}, x_2^{r})$ & $(1-\beta_1)\beta_2/(1-\alpha_1)\alpha_2$\\
         $(b, b)$ & $(1-\alpha_1)(1-\alpha_2)$ & $f(x_1^{b}, x_2^{b})$ & $(1-\beta_1)(1-\beta_2)/(1-\alpha_1)(1-\alpha_2)$\\
    \hline\hline
    \end{tabular}
    \caption{the values of $A,C,K$ for different signal profile $\mathbf{s}$}
    \label{tab: constant table}
\end{table}

Analyzing the derivatives with respect to \( u(\lambda) \), we find that the second derivative of \( \phi_\theta^{\mathbf{s}} \) with respect to \( u \) is non-negative:
\[ \frac{\partial^2 \phi_\theta^{\mathbf{s}}}{\partial u^2}(u) = \frac{2AK^2}{(K + (1-K)u)^3} \ge 0 \text{ for all } u \in (0,1). \]
This non-negative second derivative implies that each \( \phi_\theta^{\mathbf{s}}(\lambda) \) and, consequently, \( \phi_\theta(\lambda) \) is single-troughed for \( u(\lambda) \).

Finally, since \( u(\lambda) \) monotonically varies with \( \lambda \), \( \phi_\theta(\lambda) \) is single-troughed in \( \lambda \) over the interval \( (0,1] \).
\end{appendixproof}

%% file: Theory_V-shape_Regret/Theory_V-shape_supremum.tex
\begin{appendixproof}[Proof of Lemma~\ref{lemma: V-shaped supremum}]

To prove this lemma, we formally provide an equivalent definition of the single-troughed function. The definition and the equivalence relationship are shown below:


\begin{claim}[Equivalence of Single-troughed Function Definitions]\label{lemma: equivalence of definitions}
    The following two definitions of a single-troughed function over an interval \([a, b]\) are equivalent:
    \begin{enumerate}
        \item[(1)] A function \( f(x) \) is single-troughed if it is monotonically decreasing, monotonically increasing, or first decreasing then increasing over \([a, b]\).
        \item[(2)] A function \( f(x) \) is single-troughed if for all \( a \le x_1 \le x_2 \le x_3 \le b \), it holds that \( f(x_2) \le \max\{f(x_1), f(x_3)\} \).
    \end{enumerate}
\end{claim}

This claim offers a practical, verifiable method for identifying a single-troughed function, converting a conceptual understanding based on monotonicity to a more operational discriminated method.

Next, we show $\hat{f}$ is single-troughed according to the equivalent definition. Consider any $x_1, x_2, x_3 \in [a, b]$ such that $x_1 \le x_2 \le x_3$. For any $\alpha \in I$, by the single-trough of $f_\alpha$, we have
\[f_\alpha(x_2) \le \max\{f_\alpha(x_1), f_\alpha(x_3)\} \le \max\{\sup_{\alpha \in I}\{f_\alpha(x_1)\}, \sup_{\alpha \in I}\{f_\alpha(x_3)\}\} = \max\{\hat{f}(x_1), \hat{f}(x_3)\}.\]
Since $\hat{f}(x)$ is the supremum of $\{f_\alpha(x)\}$, we obtain that
\[\hat{f}(x_2) = \sup_{\alpha \in I}\{f_\alpha(x_2)\} \le \max\{\hat{f}(x_1), \hat{f}(x_3)\}.\]

Therefore, $\hat{f}(x_2)$ is always less than or equal to the maximum of $\hat{f}(x_1)$ and $\hat{f}(x_3)$, which consequently confirms that $\hat{f}$ is single-troughed in the interval $[a, b]$.
\end{appendixproof}

\begin{appendixproof}[Proof of Claim~\ref{lemma: equivalence of definitions}]
    To demonstrate the equivalence, we show that any function satisfying Definition (1) also satisfies Definition (2), and vice versa.

    \textbf{From Definition (1) to Definition (2)}: Assume \( f \) satisfies Definition (1). For any \( a \le x_1 \le x_2 \le x_3 \le b \), function \( f \) is either monotonic or first decreases then increases within \([a, b]\). In either case, it holds that \( f(x_2) \le \max\{f(x_1), f(x_3)\} \).

    \textbf{From Definition (2) to Definition (1)}: Assume \( f \) satisfies Definition (2). Let $x^*$ be the minimum point of $f$ in $[a,b]$ such that $f(x^*) \le f(x)$ for all $x \in [a,b]$. Then, for $f$ to be single-troughed, we show that $f$ is monotonically decreasing in $[a, x^*]$ and monotonically increasing in $[x^*, b]$:
    \begin{itemize}
        \item For any $a \le x_1 \le x_2 \le x^*$,  we have $f(x_2) \le \max\{f(x_1), f(x^*)\} = f(x_1)$.
        \item For any $x^* \le x_1 \le x_2 \le b$, $f(x_1) \le \max\{f(x^*), f(x_2)\} = f(x_2)$.
    \end{itemize}
\end{appendixproof}

%% file: 5_Lower_Bound.tex
\section{Lower Bound Analysis of Regrets}
We have demonstrated that for any particular aggregator $f$, the regret $R_\lambda(f)$ is single-troughed. Now we turn to study the optimal regret $\inf_{g}{R_\lambda(g)}$ and its variation trend as $\lambda$ varies. Intuitively, this optimal regret across all aggregators quantifies the distortion between the partial information, which aggregators glean from experts' predictions $\mathbf{x}(\mathbf{s}, \lambda)$, and the full information, which the omniscient aggregator acquires from the information structure $\theta$ and experts' private signals $\mathbf{s}$. 





Directly evaluating $\inf_{g}{R_\lambda(g)}$ is challenging because the optimal aggregator $g$ for each $\lambda$ value is not known. Instead, we provide an easy-to-compute lower bound, which demonstrates a V-shape. Further study in Section~\ref{section: numerical results} will show the given lower bound is almost tight. Therefore, we conjecture that the optimal regret curve $\inf_{g}{R_\lambda(g)}$ is V-shaped for $\lambda$. 

\begin{theorem}[V-shape of the Lower Bound]
For every $\lambda$, there exists a lower bound on regret, denoted as $lb(\lambda)$, such that $lb(\lambda) \le \inf_{g} R_\lambda(g)~\forall 0\le \lambda \le 1.$
This lower bound is V-shaped for $\lambda$, reaching its 
minimum value $\min_\lambda\{lb(\lambda)\} = 0$ at $\lambda = \frac{1}{2}$.
\label{theorem: unimodality of the lower bound}
\end{theorem}

Following \citet{arieli2018robust}, we build the regret lower bound by constructing two information structures, each occurring with a one-half chance.
In both structures, experts receive signals that are independent and identically distributed (i.i.d.) given the true world state $\omega$. There are two types of signals (signal $r$ or signal $b$) for each expert, i.e., $\mathcal{S}_1 = \mathcal{S}_2 = \{r, b\}$. 
We carefully design the signals so that the expert's prediction will be $\frac{1}{2}$ upon receiving signal $r$, i.e., $\brn(r, \lambda) = \frac{1}{2}$, and be either $0$ or $1$ upon receiving signal $b$, i.e., $\brn(b, \lambda) = 0/1$. 
The specifics of the two structures are outlined in Table~\ref{tab: lower bound instance}, where $\gamma$ serves as a control parameter.

\begin{table}[h]
    \centering
    \begin{tabular}{cccc}
    \hline \hline
         &priori $\mu$ & $\Pr[r\mid \omega = 1]$ & $\Pr[r\mid \omega = 0]$\\
    \hline
        Structure 1&$\gamma$ & 1& $\gamma^\lambda/(1-\gamma)^\lambda$\\
         Structure 2&$1-\gamma$ &$\gamma^\lambda/(1-\gamma)^\lambda$ & 1\\
    \hline \hline
    \end{tabular}
    \caption{Construction of Lower Bound Instance ($0 < \gamma < \frac{1}{2}$)}
    \label{tab: lower bound instance}
\end{table}

By this construction, in the case where both experts receive signal $r$, their predictions are both $\frac{1}{2}$. The likelihood of this case is the same for both structures. Therefore, an aggregator without knowledge about which structure is currently appearing can at best give an aggregated forecast at $\frac{1}{2}$. 
However, the omniscient aggregator who knows the currently occurring structure will forecast differently. The Bayesian aggregator's posterior provided by the omniscient aggregator is 
\[f^*(r,r) = 
\begin{cases} 
  \cfrac{\gamma}{\gamma + (1-\gamma)\cdot\gamma^{2\lambda}/(1-\gamma)^{2\lambda}} & \text{for Structure 1}, \\
  \cfrac{(1-\gamma)\cdot\gamma^{2\lambda}/(1-\gamma)^{2\lambda}}{(1-\gamma)\cdot\gamma^{2\lambda}/(1-\gamma)^{2\lambda} + \gamma} & \text{for Structure 2}.
\end{cases}
\]
Therefore, the regret for any aggregator $f$ is at least
\(\Pr\left[S_1 = r, S_2 = r\right] \cdot \left(\frac{1}{2}- f^*(r,r)\right)^2.\) 

Varying the parameter $\gamma$ within range $(0,\frac{1}{2})$ to maximize the above relative loss, we can build the lower bound $lb(\lambda)$ for the regret $R_\lambda$. Formally,
\[lb(\lambda) = \max_{\gamma \in (0,\frac{1}{2})}\left\{\Pr\left[S_1 = r, S_2 = r\right] \cdot \left(\frac{1}{2}- f^*(r,r)\right)^2\right\}.\]

The remaining proof of Theorem~\ref{theorem: unimodality of the lower bound} is deferred in the Appendix~\ref{app:lower bound}. We verify the V-shape of the lower bound $lb(\lambda)$ by showing the relative loss is V-shaped for any fixed parameter $\gamma$. 
Intuitively, the first component $\Pr[S_1 = r, S_2 = r]$, which is the likelihood of the indistinguishable case, decreases as $\lambda$ increases. The second component $(\frac{1}{2} - f^*(r,r))^2$, which is the gap between the best aggregation and Bayesian aggregator's posterior, first decreases and then increases, reaching a minimum value of zero at $\lambda = \frac{1}{2}$.

\begin{appendixproof}[Proof of Theorem~\ref{theorem: unimodality of the lower bound}]
\label{app:lower bound}

Substituting the likelihood for the signal profile $(r,r)$, given as $$\Pr\left[S_1 =r, S_2 = r\right] = \gamma + (1-\gamma) \gamma^{2\lambda}/(1-\gamma)^{2\lambda},$$
and the Bayesian posterior \( f^*(r,r)\) into the relative loss formula, we obtain
\begin{align*}
&\Pr\left[S_1 = r, S_2 = r\right] \cdot \left(\frac{1}{2}- f^*(r,r)\right)^2\\
= & \left(\gamma + (1-\gamma)\frac{\gamma^{2\lambda}}{(1-\gamma)^{2\lambda}}\right)\cdot \left(\frac{1}{2} - \cfrac{\gamma}{\gamma + (1-\gamma)\cdot
\gamma^{2\lambda}/(1 - \gamma)^{2\lambda}}\right)^2\\
=& \gamma\cdot \left(1 + \frac{\gamma^{2\lambda - 1}}{(1-\gamma)^{2\lambda - 1}}\right)\cdot \left(\frac{1}{2} - \cfrac{1}{1 + \gamma^{2\lambda - 1}/(1 - \gamma)^{2\lambda - 1}}\right)^2
\end{align*}

Let $y(\gamma, \lambda)$ denote value $\gamma^{2\lambda - 1}/(1 - \gamma)^{2\lambda - 1}$. We simplify the relative loss as \[ \Pr\left[S_1 = r, S_2 = r\right] \cdot \left(\frac{1}{2}- f^*(r,r)\right)^2 = \gamma\cdot\phi(y(\gamma, \lambda)),\] where \[\phi(y) = (1+y)\left(\frac{1}{2} - \frac{1}{1+y}\right)^2.\]

For any fixed $\gamma$, \( y(\gamma,\lambda)\) decreases from \( \frac{1 - \gamma}{\gamma} \) to \( \frac{\gamma}{1-\gamma} \) as \( \lambda \) increases from \( 0 \) to \( 1 \).
Notice that \( \phi(y) = (1+y)\left(\frac{1}{2} - \frac{1}{1+y}\right)^2 = \frac{1}{4} \left(\frac{4}{1+y} + (1+y)\right) - 1 \) is V-shaped for $y$ with minimum value zero reached at point \( y = 1 \). It can be derived that the relative loss decreases for \( \lambda \) from \( 0 \) to \( \frac{1}{2} \) and increases for \(\lambda\) from \( \frac{1}{2} \) to \( 1 \). 
Specifically, when $\lambda = \frac{1}{2}$, for any parameter value $\gamma$, the relative loss is zero.

Recall that the lower bound value $lb(\lambda)$ is the maximum relative loss across different parameters $\gamma$. The monotonicity for relative loss still holds for the lower bound $lb(\lambda)$.
\begin{itemize}
    \item For \( 0 \le \lambda_1 < \lambda_2 \le \frac{1}{2}\), assuming \( lb(\lambda_2) \) is obtained at \( \gamma = \gamma^* \), we have
\[ lb(\lambda_1) \ge \gamma^* \cdot \phi(y(\gamma^*, \lambda_1)) > \gamma^* \cdot \phi(y(\gamma^*, \lambda_2)) = lb(\lambda_2). \]
\item Similarly, for \(\frac{1}{2} \le \lambda_1 < \lambda_2 \le 1 \), it holds that \( lb(\lambda_1) < lb(\lambda_2) \).
\end{itemize}
\end{appendixproof}


%% file: 7_Experiment.tex
\section{Study}\label{study}

We have theoretically and numerically assessed the performance of different aggregators. 
Regarding aggregating predictions from real-world human subjects, we investigate the following questions:

\begin{itemize}
\item [(1)] Do people display base rate neglect as prior empirical studies suggest?

\item [(2)] Which aggregator is best for aggregating predictions empirically?

\item [(3)] Does a certain degree of base rate neglect help aggregation in practice?
\end{itemize}

To further examine these questions, we conduct an online study to identify base rate neglect in human subjects and empirically compare our aggregators with alternatives. To make our comparison more representative, we use average loss rather than worst-case loss to measure aggregators' performance. Our findings are outlined below:

\begin{itemize}
\item [(1)] 
Types of Responses: Very few predictions are perfect Bayesian. Some of them display base rate neglect.
However, around ~57\% of predictions do not fall between perfect BRN and Bayes, which is beyond our theoretical base rate neglect model. Around ~19\% even just report the prior, indicating a tendency opposite to base rate neglect, which exhibits signal neglect. 

\item [(2)] 
New Aggregator Wins in Inside Group:
Among the general population, simple averaging achieves the lowest average loss in the level of information structure. 
This is because ~57\% of predictions fall outside the perfect BRN-Bayes range that our theoretical model considers. When we restrict the predictions within this range,
certain $\hat{\lambda}$-base rate balancing aggregators with $\hat{\lambda} < 1$ can achieve lower loss than other aggregators, such as simple average and average prior, aligning with theoretical results in previous sections.

\item [(3)] Base Rate Neglect Helps Aggregation: Within the same aggregator, some degree of base rate neglect does not necessarily hurt forecast aggregation - it may even improve it. 
\end{itemize}

The following content of this section presents the design and results of our study in detail.
We highlight that different from previous studies which only focus on several specific information structures [e.g., \citealp{ginossar1987problem,esponda2023mental}], our work collects a comprehensive dataset on predictions under tens of thousands of information structures. 


\subsection{Study Design and Data Collection}
\label{section:design}

\paragraph{\textbf{Task}} We use the standard belief-updating task to elicit the forecast of subjects \citep{phillips1966conservatism,grether1980bayes}. Specifically, there are two boxes, each containing a mix of red and blue balls with a total of 100. In the left box, the proportions of red and blue balls are $p_{le}\in (0,1)$ and $(1-p_{le})$ respectively. Similarly, the proportions in the right box are $p_{ri}\in (0,1)$ and $(1-p_{ri})$. 
One box is selected randomly. 
Particularly, the probability of selecting the left box is $\mu\in(0,1)$, and that of the right box is $1-\mu$. Then one ball is randomly drawn from the selected box. The color of the drawn ball is informed to the subjects as a signal.
After knowing the signal, subjects are required to estimate the probability that the drawn ball comes from the left box\footnote{Specifically, the two questions are ``If the ball is red, what is the probability that it comes from the left box'' and ``If the ball is blue, what is the probability that it comes from the left box''.}. 
We consider a finite set of information structures, and name the specific combinations of the parameters (i.e., $\mu, p_{le}, p_{ri}$) as cases. The parameters are all multiples of one tenth. 
Consequently, there are $9^3=729$ cases in total. Each subject is required to answer 30 different cases. 
In each case, the subject's predictions upon two signals (red ball or blue ball) are collected \footnote{To ensure subject's predictions upon the two signals in the same case are independent, we assign them randomly across different rounds.}.
Therefore, each subject should answer 60 rounds of questions involving 30 cases. 
Predictions should be stated in percentage points, with values ranging from 0\% to 100\%.

\paragraph{\textbf{Procedure}} The experiment is conducted using Otree \citep{chen2016otree} and we recruit a balanced sample of male and female from Prolific \citep{Palan2018prolific}.
Subjects provide informed consents and are made aware that their responses would be used for research purposes. We use the incentive-compatible BDM method \citep{becker1964measuring} to elicit their true belief.
 Particularly, we introduce the example task and payment scheme before the formal task to guarantee subjects' understanding. Appendix~\ref{appendix:instruction} shows the instructions to subjects.

At last, 291 subjects finished the study. On average, there are 11.98 different subjects providing predictions under each case. In total, we obtain predictions under 29,889 information structures. The experiment lasts 32.68 minutes on average and the average payment is around \$8.16 (including \$5.5 participation fee).
\hide{At last, 201 subjects finished the study, and we exclude x of them who spent less than x minutes on the study, resulting in x subjects and x predictions used for the formal analyses. Thus, on average, there are x predictions under each information structure. The average duration for the remaining subjects is 33.7 minutes and the average payment is \$x.  }


\subsection{Identification of Base Rate Neglect}\label{subsection:brn}
Our work is motivated by the well-observed deviation from Bayesian predictions.
Therefore, our first objective is to identify whether the responses from subjects in our study align with Bayesian principles. 
Thus, we use the perfect Bayesian posterior as benchmark: for a red ball signal, it is $x^{Bayes}(r) = \frac{\mu  p_{le}}{\mu   p_{le} + (1-\mu)   p_{ri}}$, and for a blue ball signal, it is $x^{Bayes}(b) = \frac{\mu   (1-p_{le})}{\mu   (1-p_{le}) + (1-\mu)   (1-p_{ri})}$. 
We name these reports as perfect Bayes. Furthermore, similar to \citet{esponda2023mental}, we define responses $x^{pBRN}(r) = \frac{p_{le}}{p_{le}+p_{ri}}$ and $x^{pBRN}(b) = \frac{1-p_{le}}{(1-p_{le})+(1-p_{ri})}$ and name them as perfect Base Rate Neglect (perfect BRN), which corresponds to the instance of $\lambda = 0$ in our base rate neglect model.


\paragraph{\textbf{Base Rate Neglect at Prediction Level}} The results of our study show that only 12.44\% of the predictions are consistent with perfect Bayes\footnote{We exclude the cases when $\mu = 0.5$ during the analyses in this subsection, because there is no difference between $x^{Bayes}$ and $x^{pBRN}$.}$^{,}$\footnote{We relax the Bayesian belief by permitting rounding $x^{Bayes}$ both up and down to two decimal places, and the same principle applies to $x^{pBRN}$. For example, both 0.56 and 0.57 are regarded as perfect Bayes when the actual Bayesian posterior is 0.5625.}.
Meanwhile, 5.37\% of the predictions fully ignore the base rate, which is consistent with perfect BRN.
Moreover, 25.11\% of the responses fall inside the range between $x^{Bayes}$ and $x^{pBRN}$ (named inside group), which exhibit partial base rate neglect, and 57.08\% fall outside (named outside group)\footnote{We acknowledge that our theoretical model does not encompass predictions in outside group. Nevertheless, our subsequent empirical analyses will incorporate such predictions and examine how aggregators perform when aggregating them.}$^{,}$\footnote{In our study, the occurrence of perfect BRN is relative low when compared to what is documented in existing literature. For example, using \citet{kahneman1972prediction}'s taxicab problem, \citet{bar1980base} finds around 10\% subjects provide Bayesian predictions while 36\% fully ignore the base rate. The low occurrence in our study can be ascribed to two factors.
Firstly, our study introduces a broader range of information structures beyond the classical cases known to easily provoke base rate neglect.
Secondly, we use abstract description instead of contextualized vignette for simplicity, leading to a lower degree of base rate neglect \citep{ganguly2000asset}.}. 

\hide{When we narrow our focus to the predictions in the first round instead of all 60 rounds, where learning effect is not a concern, the proportion of perfect Bayes slightly decreases, while that of perfect BRN slightly increases, both reaching 6.99\%. \yk{do not understand this sentence, red ball round?}\shu{more clear now maybe.} }

Figure~\ref{fig: ans_ratio} shows the proportion of different types of responses conditional on the signal across rounds. We observe that these proportions are relatively stable with respect to both rounds and signals.
Thus, we combine the reports under two signals in the following analyses. The above findings together validate that subjects rarely submit Bayesian beliefs. 

In addition, we also notice that there are 18.94\% of the predictions consistent with the priors, which is also a type of systematic deviation from Bayesian reasoning named signal neglect \citep{phillips1966conservatism,coutts2019good,campos2023non}. 


\begin{figure}[h]
    \centering
    \includegraphics[width=0.8\linewidth]{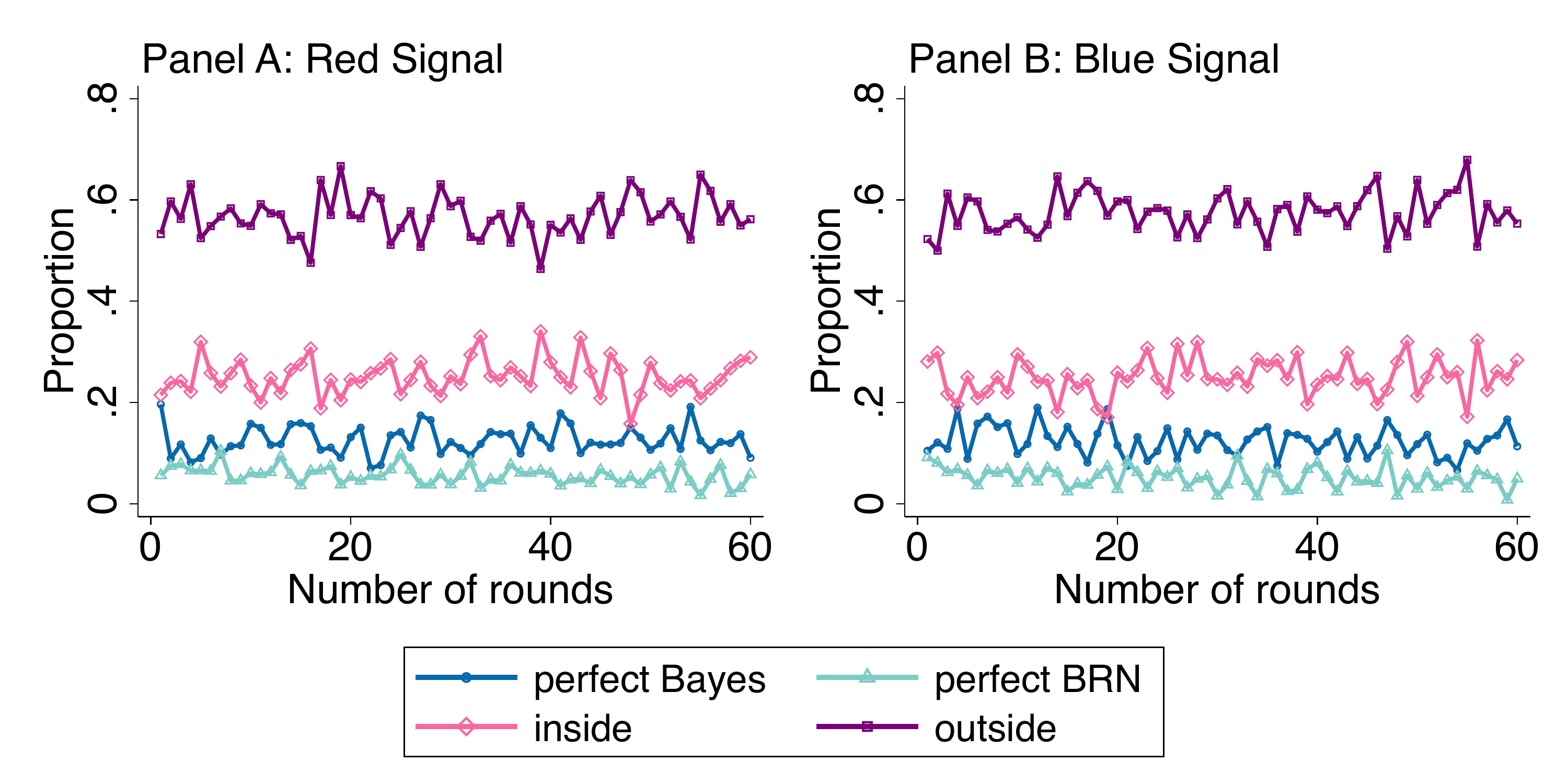}
    \caption{Proportions of Different Types of Response}
    \label{fig: ans_ratio}
\begin{flushleft}
{\footnotesize \textit{Notes:} We exclude the responses when $\mu=0.5$ because there is no difference between $x^{Bayes}$ and $x^{pBRN}$. Panel A and B correspond to the responses when the signal is red ball and blue ball, respectively. \hide{Red solid line denotes the proportion of responses that equal to $B_{Bay}$, blue dash line denotes that equal to $B_{pBRN}$, green dash-dot line indicates that fall between the two, and brown line indicates that fall outside.}}
\end{flushleft}
\end{figure}

\paragraph{\textbf{Base Rate Neglect at Subject Level}} After exploring base rate neglect at prediction level, another question arises: how far do subjects deviate from Bayesian? To answer this question, we estimate the base rate consideration degree $\lambda$ for each subject $i$. According to Observation \ref{lemma: BNR prediction vs Bayesian}, we can obtain the following econometric model, 
\[\logit (x^{Bayes}_{(t)}) - \logit (x_{(t)}) =\beta  \logit(\mu_{(t)}) + \varepsilon_{(t)}, \]
where $x_{(t)}$ is subject's prediction in round $t$, $x^{Bayes}_{(t)}$ is the corresponding perfect Bayes benchmark, 
and $\logit(x)$ represents the log odds function where $\logit(x) = \log \frac{x}{1-x}$. The coefficient $\beta$ is of our interest, and equation $\lambda = 1- \beta$ holds. We estimate the above econometric model and obtain estimated $\lambda$
using Ordinary Least Squares (OLS) regression 
for each subject. 

Figure \ref{fig: est_lambda} depicts the distribution of estimated $\lambda_i$. The results show that $\lambda_i$ displays three distinct peaks
corresponding to 0 (perfect BRN), 0.6 (representing moderate BRN), and 1 (perfect Bayes), respectively.
The average consideration degree of the base rate at subject level is 0.4488.
Besides, a minority of subjects have prior consideration degree $\lambda_i$ that falls outside the range of $[0,1]$. The proportion of such subjects is relatively small and these deviations are minor, with only 3.09\% of $\lambda_i$ values being less than -0.2, and none exceeding 1.2. 

\begin{figure}[h]
    \centering    \includegraphics[width=0.5\linewidth]{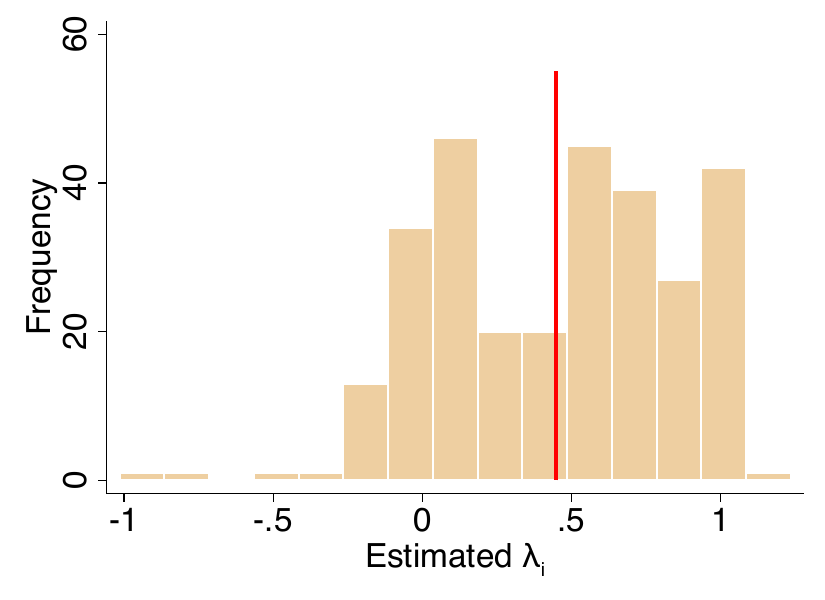}
    \caption{The Distribution of Estimated $\lambda_i$}
    \label{fig: est_lambda}
\begin{flushleft}
{\footnotesize \textit{Notes:} 
We exclude the responses when $\mu=0.5$ because there is no difference between $x^{Bayes}$ and $x^{pBRN}$. The red line indicates the average estimated $\lambda_i$ at subject level.}
\end{flushleft}
\end{figure}


\subsection{Aggregator evaluation}
After observing base rate neglect, we further explore the performance of aggregators under subjects' predictions. We denote $\zeta$ as  a single-expert information structure with parameter $(p_{le},p_{ri},\mu)$ provided in the task. Each single-expert structure $\zeta$ corresponds to a case in our study. Subject $i$'s predictions upon the signals of red ball and blue ball under case $\zeta$ are denoted as $x^{\zeta}_i(r)$ and $x^{\zeta}_i(b)$.


Combining two single-expert information structures with the same selection probability $\mu$ of the left box, we can obtain an information structure defined in Section~\ref{section: problem statement} where experts' signals are independent conditional on the selected box. Let $\zeta_1$, $\zeta_2$ be two single-expert information structures with same parameter $\mu$. The event that the left box is selected corresponds to the state $\omega = 1$. 
Then the state distribution in the combined information structure $\theta$ is $\Pr[\omega = 1] = \mu$ and $\Pr[\omega = 0] = 1- \mu$. The conditional distributions of experts' signals are 
\begin{equation*}
    \begin{aligned}
        &\Pr[S_i = r\mid \omega = 1] = p_{le}(\zeta_i),~\Pr[S_i = b\mid \omega = 1] = 1-p_{le}(\zeta_i) ~\forall i = 1,2 , \\
        &\Pr[S_i = r \mid \omega = 0] = p_{ri}(\zeta_i),~ \Pr[S_i = b \mid \omega = 0] = 1 - p_{ri}(\zeta_i)~\forall i = 1,2 .
    \end{aligned}
\end{equation*}

Let $\mathcal{I}_{\zeta_1}$, $\mathcal{I}_{\zeta_2}$ denote the set of subjects assigned with case $\zeta_1$, $\zeta_2$ respectively.
The empirical relative loss of the combined information structure $\theta$ under aggregator $f$ is defined as

\[Loss_{\theta}^{f} = \sum\limits_{ i_1\in \mathcal{I}_{\zeta_1},i_2\in\mathcal{I}_{\zeta_2},i_1\neq i_2}{C_{\zeta_1, \zeta_2}}^{-1} \sum\limits_{s_1, s_2 \in \{ r, b \}}{
    \Pr[S_1 = s_1, S_2 = s_2]
\left[ f \left( x_{i_1}^{\zeta_1}(s_1),  x_{i_2}^{\zeta_2}(s_2)\right) - f^*\left(s_1, s_2\right) \right]^2},
\]
where \(C_{\zeta_1, \zeta_2}= \sum_{i_1,i_2}\mathbf{1}[i_1\in \mathcal{I}_{\zeta_1}, i_2\in\mathcal{I}_{\zeta_2},i_1\neq i_2]\) and the Bayesian aggregator's posterior is \( f^*(s_1,s_2)= \Pr[\omega = 1\mid S_1 = s_1, S_2 =s_2].\)

Intuitively, the empirical loss $Loss_\theta^f$ is determined by averaging the losses across all possible pairs of subjects' predictions. This empirical loss is exactly the expected square loss if we randomly choose two subjects who are assigned with case $\zeta_1$ and $\zeta_2$ respectively, select the box according to $\mu$, and draw balls for subjects following the probability given by $p_{le}$ and $p_{ri}$.
We emphasize that in order to ensure the independence of predictions and to avoid aggregating two predictions from the same subject, we exclude instances where $i=j$\footnote{Namely, predictions from a single subject will not be aggregated.}. Thus, we construct a dataset including real-world human predictions under each pair of cases $(\zeta_1,\zeta_2)$, which enables us to formally evaluate the performance of aggregators. To calculate the relative loss of aggregators when inputting Bayesian posteriors, we substitute subjects' predictions by perfect Bayes $x^{Bayes}$.




\paragraph{\textbf{Whole Sample Analyses}}
Tables \ref{table:agg perf} summarizes the performance of our $\hat{\lambda}$-base rate balancing aggregators $(\hat{\lambda}<1)$, average prior and simple average on aggregating subjects' predictions. 
We find that when $\hat{\lambda}$ ranges from 0.1 to 0.9, there is a decrease in both average loss and maximum loss. However, despite this decrease, all these aggregators achieve higher loss than both average prior and simple average aggregators, with simple average performing the best.

This pattern shifts when aggregating Bayesian posteriors. While the trend concerning changes in $\hat{\lambda}$ remains consistent, the $\hat{\lambda}$-base rate balancing aggregator with $\hat{\lambda}=0.9$, surpasses the average prior in terms of average loss. In addition, the simple average aggregator only demonstrates moderate performance.


\hide{
\begin{figure}[h]
    \centering
    \includegraphics[width=\linewidth]{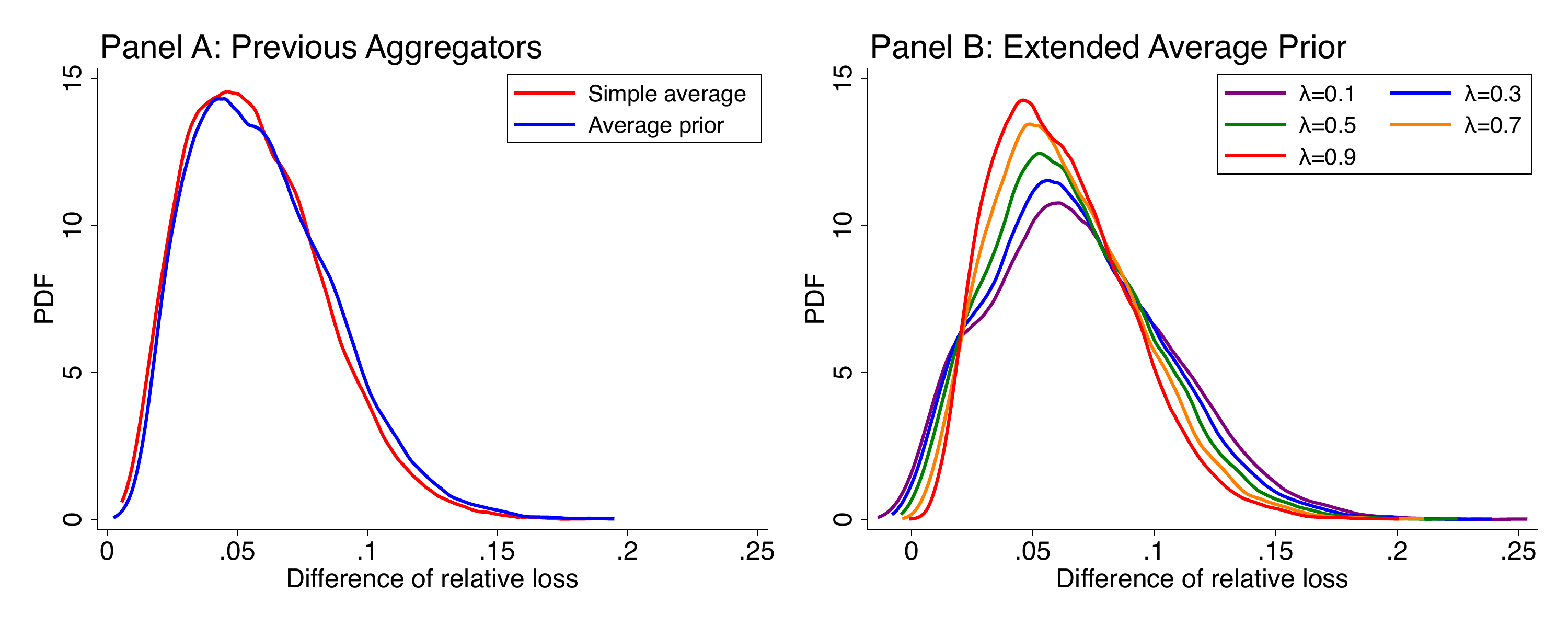}
    \caption{PDF of the Difference of Loss between Subjects' predictions and Bayesian Posteriors}
    \label{fig: diff_density}
\begin{flushleft}
{\footnotesize \textit{Notes:} \shu{change the color as figure 4\&7} }
\end{flushleft}
\end{figure}
}

\setlength{\tabcolsep}{4pt}
\begin{table}
\footnotesize
\begin{threeparttable}   
\begin{tabular}{lccccccccccc}
\hline \hline
 & \multicolumn{9}{c}{$\hat{\lambda}$-base rate balancing aggregator $f^{\hat{\lambda}}_{ap}$ ($\hat{\lambda} < 1$)} & \multirow{2}{*}{\makecell{Average \\ prior}} & \multirow{2}{*}{\makecell{Simple \\ average}} \\
  \cline{2-10}
 \multicolumn{1}{l}{} & $\hat{\lambda}=0.1$ & $\hat{\lambda}=0.2$ & $\hat{\lambda}=0.3$ & $\hat{\lambda}=0.4$ & $\hat{\lambda}=0.5$ & $\hat{\lambda}=0.6$ & $\hat{\lambda}=0.7$ & $\hat{\lambda}=0.8$ & $\hat{\lambda}=0.9$ &  &  \\
 \hline
Avg. loss   & 0.0882 & 0.0853 & 0.0823 & 0.0793 & 0.0762 & 0.0731 & 0.0702 & 0.0675 & 0.0652 & 0.0638 & 0.0627 \\
Max. loss  & 0.2929 & 0.2863 & 0.2792 & 0.2714 & 0.2630 & 0.2539 & 0.2443 & 0.2341 & 0.2269 & 0.2203 & 0.2155 \\
\hline \hline
\end{tabular}
\begin{tablenotes}[flushleft]
\footnotesize \item \textit{Notes:} The number of observations is 29,889. For the convenience of comparison, we exclude 44 pairs of predictions that cannot be aggregated by either average prior aggregator or $\hat{\lambda}$-base rate balancing aggregators ($\lambda < 1$). This exclusion applies to cases where, for instance, one subject reports a probability of 0\% while another reports 100\%. 
\end{tablenotes}
\end{threeparttable}
\caption{Summary of Aggregators' Performance on Subjects' Predictions}
\label{table:agg perf}
\end{table}


\paragraph{\textbf{Subsample Analyses}}
We note that there exists a gap between our theoretical results and the above empirical analyses. Theoretically, the consideration degree of base rate $\lambda$ is assumed to vary between 0 and 1, which means the actual predictions should lie between the extremes of perfect BRN and perfect Bayes. However, as depicted in Figure~\ref{fig: ans_ratio}, around 57\% of the predictions fall outside this expected range.

To close this gap and gain deeper insights, we categorize the sample based on whether the predictions fall within the expected range. We then investigate the heterogeneous performance of our $\hat{\lambda}$-base rate balancing aggregators, particularly focusing on the predictions that do locate within the perfect BRN - perfect Bayes range, which we refer to as the inside group.


As mentioned in Subsection~\ref{subsection:brn}, there are two main groups of predictions: those within and outside the expected range. 
In our context, we aggregate predictions from two experts, each providing two predictions based on the received signals. We first identify five subsamples according to the composition of these four reports, ranging from the subsample where all four reports are outside the expected range (4 outside) to that where all four reports are inside it (4 inside).
Additionally, we consider two special instances: one where all four reports are perfect BRN (4 perfect BRN) and another where all reports are perfect Bayes (4 perfect Bayes).
Figure~\ref{fig: heter_agg} shows the performance of various aggregators across the above subsamples, assessed in terms of average loss at information structure level.

\input{picture/heter_agg_tex}

For the subsample of 4 outside reports and that of 1 inside and 3 outside reports, the simple average aggregator achieves the lowest average loss. However, this pattern does not hold for the other instances. 
For subsamples where most of the reports exhibit base rate neglect, certain $\hat{\lambda}$-base rate balancing aggregators with $\hat{\lambda} < 1$ surprisingly benefit the aggregation. Notably, the $\hat{\lambda}$-base rate balancing aggregator with $\hat{\lambda} =0.7$ performs best for the subsample of 3 inside and 1 outside reports, while the \hide{$\hat{\lambda}$-base rate balancing }aggregator with $\hat{\lambda} =0.5$ is optimal for both 4 perfect BRN reports and 4 inside reports.
In contrast, for the subsample of 2 inside and 2 outside reports, as well as the subsample of 4 perfect Bayes reports, the average prior is most effective. The above findings underscore the critical importance of choosing the appropriate aggregator based on experts' consideration degree of the base rate, which can significantly improve the aggregation accuracy. 

\subsection{Base Rate Neglect vs. Bayesian}
At last, we investigate the role of base rate neglect in forecast aggregation. Namely, given the same aggregator, we study whether the prior consideration degree influences the performance of aggregators. Given validation that subjects do not submit Bayesian posteriors, we compare the performance of an aggregator across two distinct scenarios to answer this question, where the first involves subjects' actual reports and the second considers hypothetical Bayesian posteriors.


\paragraph{\textbf{Whole Sample Analyses}}
Existing aggregators including simple average and average prior achieve higher loss for human subjects, with the average loss being 0.0627 and 0.0638 respectively (see Table \ref{table:agg perf}). This loss significantly reduces to 0.0091 and 0.0076 when Bayesian posteriors are used for aggregation (see Table~\ref{table:agg perf appendix} in Appendix~\ref{appendix:analyses}). Moreover, Bayesian posteriors consistently enhance the aggregation accuracy in all tested information structures.



As for $\hat{\lambda}$-base rate balancing aggregators, subjects' reports also result in worse performance in the general population of tested information structures. 
Interestingly, as $\hat{\lambda}$ increases, the loss difference between aggregating subjects' reports and aggregating Bayesian posteriors diminishes, suggesting that the performance of $\hat{\lambda}$-base rate balancing aggregator gets better as $\hat{\lambda}$ increases when inputting subjects' predictions.
However, the proportion of structures that subjects' predictions result in a lower loss than Bayesian posteriors decreases from 0.71\% to 0.00\% as $\hat{\lambda}$ increases from 0.1 to 0.9. This diminishing trend becomes even larger, from 25.20\% to 7.36\%, when examining the loss \hide{from subjects' predictions }at the prediction pairs level (see Table~\ref{table:agg perf appendix} in Appendix~\ref{appendix:analyses}). This highlights that non-Bayesian predictions may result in better aggregation compared to Bayesian ones. 

\paragraph{\textbf{Subsample Analyses}}
When comparing aggregators' performance across subsamples (see Figure~\ref{fig: heter_agg}), we find that subsample of 4 perfect Bayes reports does not always achieve the lowest loss for all aggregators. Our aggregators with $\hat{\lambda} \le 0.7$ can achieve lower loss when aggregating 4 inside reports, compared to that when aggregating 4 perfect Bayesian reports. 
Moreover, the minimal loss across all tested aggregators when aggregating 4 inside reports ($\hat{\lambda}$-base rate balancing aggregator with $ \hat{\lambda} = 0.5$, $loss = 0.0063$) is less than that when aggregating 4 perfect Bayesian reports (average prior, $loss = 0.0068$). This observation implies that base rate neglect does not necessarily compromise the aggregation performance, which is consistent with our theoretical results.








%% file: picture/heter_agg_tex.tex
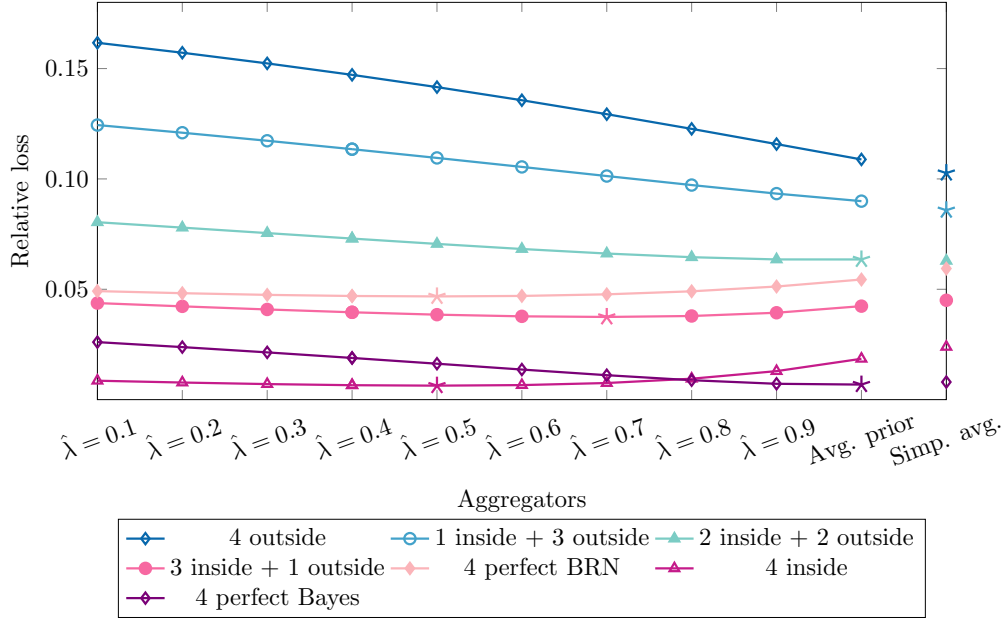
\begin{figure}[h]
    \centering
    \begin{tikzpicture}[scale=0.85] 
    \begin{axis}[
        width=0.9\textwidth, 
        height=7.8cm, 
        xlabel={Aggregators},
        ylabel={Relative loss},
        xmin=1, xmax=11,
        xtick={1,2,3,4,5,6,7,8,9,10,11}, 
        xticklabels={$\hat{\lambda}=0.1$, $\hat{\lambda}=0.2$, $\hat{\lambda}=0.3$, $\hat{\lambda}=0.4$, $\hat{\lambda}=0.5$, $\hat{\lambda}=0.6$, $\hat{\lambda}=0.7$, $\hat{\lambda}=0.8$, $\hat{\lambda}=0.9$, Avg. prior, Simp. avg.} , 
        ytick={0.05,0.1,0.15}, 
        yticklabels={0.05,0.10,0.15} , 
        xticklabel style={rotate = 20},
        ymin=0, ymax=0.18,
        legend style={at={(0.5,-0.3)},
	       anchor=north,legend columns=3},
        grid style=dashed,
        ]

    
    \addplot[color=blue4,
    solid,
    line width=1pt,
    mark=diamond,
    mark options={solid},
    mark size=2.5pt] coordinates {
        (1,0.161681)(2,0.157180)(3,0.152347)(4,0.147158)
    };
    \addlegendentry{4 outside}
    \addplot[color=blue4,
    solid,
    line width=1pt,
    mark=none,
    mark options={solid},
    mark size=2.5pt,
    forget plot] coordinates {
    (4,0.147158)(5,0.141596)(6,0.135652)(7,0.129333)(8,0.122679)(9,0.115782)(10,0.108827)
    };    
    \addplot[only marks, mark=diamond,color=blue4,mark size=2.5pt,line width=1pt, forget plot] coordinates {
    (5,0.141596)(6,0.135652)(7,0.129333)(8,0.122679)(9,0.115782)(10,0.108827)
    };
    \addplot[only marks, mark=star,color=blue4,mark size=4pt,line width=1pt, forget plot] coordinates {
     (11,0.102643)
    };    

    \addplot[color=blue3,
    solid,
    line width=1pt,
    mark=o,
    mark options={solid},
    mark size=2.5pt] coordinates {
        (1,0.124410)(2,0.120937)(3,0.117291)(4,0.113478)
    };
    \addlegendentry{1 inside + 3 outside}
    \addplot[color=blue3,
    solid,
    line width=1pt,
    mark=none,
    mark options={solid},
    mark size=2.5pt,forget plot] coordinates {
        (4,0.113478)(5,0.109516)(6,0.105440)(7,0.101309)(8,0.097223)(9,0.093344)(10,0.089931)
    };
    \addplot[only marks, mark=o,color=blue3,mark size=2.5pt,line width=1pt,forget plot] coordinates {
     (4,0.113478)(5,0.109516)(6,0.105440)(7,0.101309)(8,0.097223)(9,0.093344)(10,0.089931)
    };    
    \addplot[only marks, mark=star,color=blue3,mark size=4pt,line width=1pt,forget plot] coordinates {
     (11,0.085808)
    };

    \addplot[color=blue2,
    solid,
    line width=1pt,
    mark=triangle*,
    mark options={solid},
    mark size=2.5pt] coordinates {
    (1,0.080388)(2,0.077940)(3,0.075459)(4,0.072979)
};
    \addlegendentry{2 inside + 2 outside}
    \addplot[color=blue2,
    solid,
    line width=1pt,
    mark=none,
    mark options={solid},
    mark size=2.5pt, forget plot] coordinates {
   (4,0.072979)(5,0.070552)(6,0.068252)(7,0.066192)(8,0.064535)(9,0.063524)(10,0.063512)
};
    \addplot[only marks, 
    mark=triangle*,color=blue2,mark size=2.5pt,
    line width=1pt, forget plot] coordinates {
   (4,0.072979)(5,0.070552)(6,0.068252)(7,0.066192)(8,0.064535)(9,0.063524)(11,0.062884)
    };        
    \addplot[only marks, 
    mark=star,color=blue2,
    mark size=4pt,
    line width=1pt, forget plot] coordinates {
     (10,0.063512)
    };       

    \addplot[color=pink2,
    solid,
    line width=1pt,
    mark=*,
    mark options={solid},
    mark size=2.5pt] coordinates {
    (1,0.043743)(2,0.042262)(3,0.040854)(4,0.039572)
};
    \addlegendentry{3 inside + 1 outside}
    \addplot[color=pink2,
    solid,
    line width=1pt,
    mark=none,
    mark options={solid},
    mark size=2.5pt,forget plot] coordinates {
(4,0.039572)(5,0.038497)(6,0.037740)(7,0.037462)(8,0.037893)(9,0.039367)(10,0.042356)
};
    \addplot[only marks, 
    mark=*,
    color=pink2, mark size=2.5pt,
    line width=1pt, forget plot] coordinates {
(4,0.039572)(5,0.038497)(6,0.037740)(8,0.037893)(9,0.039367)(10,0.042356)(11,0.045018)
};  
    \addplot[only marks, 
    mark=star,
    color=pink2, mark size=4pt,
    line width=1pt, forget plot] coordinates {
     (7,0.037462)
    };         
    
    \addplot[color=pink1,
    solid,
    line width=1pt,
    mark=diamond*,
    mark options={solid},
    mark size=2.5pt] coordinates {
        (1,0.049172)(2,0.048223)(3,0.047466)(4,0.046958)
    };
    \addlegendentry{4 perfect BRN}  
    \addplot[color=pink1,
    solid,
    line width=1pt,
    mark=none,
    mark options={solid},
    mark size=2.5pt,forget plot] coordinates {
       (4,0.046958)(5,0.046768)(6,0.046983)(7,0.047709)(8,0.049077)(9,0.051250)(10,0.054421)
    };
    \addplot[only marks, 
    mark=diamond*,
    color=pink1,mark size=2.5pt,
    line width=1pt, forget plot] coordinates {
       (4,0.046958)(6,0.046983)(7,0.047709)(8,0.049077)(9,0.051250)(10,0.054421)(11,0.059402)
    };  
    \addplot[only marks, 
    mark=star,
    color=pink1,mark size=4pt,
    line width=1pt, forget plot] coordinates {
     (5,0.046768)
    };  

    \addplot[color=pink3,
    solid,
    line width=1pt,
    mark=triangle,
    mark options={solid},
    mark size=2.5pt] coordinates {
        (1,0.008572)(2,0.007748)(3,0.007046)(4,0.006538)
    };
    \addlegendentry{4 inside}
    \addplot[color=pink3,
    solid,
    line width=1pt,
    mark=none,
    mark options={solid},
    mark size=2.5pt,forget plot] coordinates {
         (4,0.006538)(5,0.006334)(6,0.006588)(7,0.007527)(8,0.009479)(9,0.012916)(10,0.018508)
    };
    \addplot[only marks, 
    mark=triangle,
    color=pink3,mark size=2.5pt,
    line width=1pt, forget plot] coordinates {
         (4,0.006538)(6,0.006588)(7,0.007527)(8,0.009479)(9,0.012916)(10,0.018508)(11,0.023898)
    };    
    \addplot[only marks, 
    mark=star,
    color=pink3,mark size=4pt,
    line width=1pt, forget plot] coordinates {
      (5,0.006334)
    };

    \addplot[color=pink4,
    solid,
    line width=1pt,
    mark=diamond,
    mark options={solid},
    mark size=2.5pt] coordinates {
        (1,0.026061)(2,0.023804)(3,0.021409)(4,0.018890)
    };
    \addlegendentry{4 perfect Bayes}
    \addplot[color=pink4,
    solid,
    line width=1pt,
    mark=none,
    mark options={solid},
    mark size=2.5pt,forget plot] coordinates {
        (4,0.018890)(5,0.016278)(6,0.013634)(7,0.011074)(8,0.008801)(9,0.007175)(10,0.006820)
    };
    \addplot[only marks, 
    mark=diamond,
    color=pink4,mark size=2.5pt,
    line width=1pt, forget plot] coordinates {
        (4,0.018890)(5,0.016278)(6,0.013634)(7,0.011074)(8,0.008801)(9,0.007175)(11,0.007971)
    };  
    \addplot[only marks, 
    mark=star,
    color=pink4,mark size=4pt,
    line width=1pt, forget plot] coordinates {
      (10,0.006820)
    };     
    
    \end{axis}
\end{tikzpicture}
    \caption{Aggregators' Performance in Subsample}
    \label{fig: heter_agg}
    \begin{flushleft}
{\footnotesize \textit{Notes:} For the convenience of comparison, we exclude 44 pairs of predictions that cannot be aggregated by either average prior aggregator or $\hat{\lambda}$-base rate balancing aggregators ($\lambda < 1$). This exclusion applies to cases where, for instance, one subject reports a probability of 0\% while another reports 100\%. The symbol * denote the lowest loss across aggregators within the subsample.
}
\end{flushleft}
\end{figure}

%% file: 8_Conclusion.tex
\section{Conclusion}

This work provides a first step to consider robust forecast aggregation when experts exhibit base rate neglect. We theoretically illustrate the single-troughed regret regarding the consideration degree of base rate $\lambda$ and examine it numerically. Moreover, we construct a family of $\hat{\lambda}$-base rate balancing aggregators that take experts' base rate consideration degree into account.
We also numerically show that the aggregator with an appropriate $\hat{\lambda}$ can achieve a low regret across all possible degree $\lambda$.
To justify the validity of those findings, we also collect a comprehensive dataset of predictions from human subjects under various information structures from an online study.

There are some limitations to our work. First, as a starting point, we only consider the scenario for aggregating predictions from two experts, and they are assumed to exhibit the same consideration degree of base rate. Although we relax the latter assumption in our empirical study, we believe it is interesting to theoretically explore the cases of aggregating predictions from multiple experts with heterogeneous consideration degrees of base rate. Second, as our empirical study reveals, some people have base rate neglect but some have signal neglect. A richer theoretical framework should be considered to incorporate both situations. Finally, an experiment with real-world scenarios where the base rates and private signals are not explicitly presented to the subjects is worth studying. 

%% file: appendix_experiment.tex
\section{Instruction}\label{appendix:instruction}

In this appendix, we present the instruction used in the online study. First, the subjects receive consent forms, followed by a straightforward coin flipping exercise designed to familiarize them with the task and the payment scheme \citep{esponda2023mental}. Subsequently, we introduce a sample task and explain it, then 60 rounds of formal task start. In addition, we also ask several questions about their demographics.

In the following instruction, the content in \{ \} varies across subjects or rounds. Comments for clarity are provided in brackets [ ] and are italicized, which are not visible to subjects during the study. The question that require a response are marked by dot ($\bullet$).  Note: We use \underline{underline} to replace personal information about authors in this instruction.

\vspace{0.3cm}

{\large \textbf{Welcome!}} [\textit{New page}]

$\bullet$ \ Please enter your Prolific ID.

\textbf{Contact}

This study is conducted by a research team in \underline{ \quad \quad \quad \quad } [\textit{authors' universities}], \underline{ \quad\quad\quad \quad } [\textit{authors' country}].
If you have any questions, concerns or complaints about this study, its procedures, risks, and benefits, please write to \underline{ \quad \quad \quad \quad } [\textit{one author's email}]

\textbf{Confidentiality}

This study is anonymous. The data collected in this study do not include any personally identifiable information about you. By participating, you understand and agree that the data collected in this study will be used by our research team and aggregated results will be published.

\textbf{Duration}

This study lasts approximately 40 minutes. 

You may choose to stop participating in this study at any time, but you cannot gain any payment.

\textbf{Qualification}

A set of instructions will be given at the start. Please read the instructions carefully. The formal task consists of 60 rounds of questions about decision-making. 
Please do not talk with others or search the answers on the Internet.

\textbf{Payment}

You will receive \$5.5 as participation fee if you finish the whole study. We will randomly select 1 round in the formal task to pay you an additional bonus, which will be either \$6 or \$0. The likelihood of getting \$6 will be determined by your choice in the selected round. The transfer of bonuses will take up a week. 

By ticking the following boxes, you indicate that you understand and accept the rules, and you would like to participate in this study. 

$\square$ I understand and accept the rules, and I would like to participate in this study.

$\square$ I am above 18 years old.

\vspace{0.3cm}

{\large \textbf{Basic Questions}} [\textit{New page}]

$\bullet$ \ 1) What is your gender? a) Male, b) Female.

$\bullet$ \ 2) What is your age? a) <18 years old, b) 18-24 years old, c) 25-34 years old, d) 35-44 years old, e) 45-54 years old, f) 55-64 years old, g) >=65 years old.

$\bullet$ \ 3) What is your race? a) American Indian or Alaska Native, b) Asian, c) Black or African American, d) Native Hawaiian or Other Pacific Islander, e) White, f) Others.

$\bullet$ \ 4) What is your nationality? a) American, b) Indian, c)  Canadian, d) Others.

$\bullet$ \ 5) What is your educational level? a) Elementary school, b) High school, c) Associate's, d) Bachelor's, e) Master's, f) Ph.D.

$\bullet$ \ 6) What is your current employment status? a) Employed full time (40 or more hours per week), b) Employed part time (up to 39 hours per week), c) Umemployed and currenctly looking for work, d) Unemployed and not currently looking for work, e) Student, f) Retired, g) Homemaker, h) Self-employed, g) Unable to work.

$\bullet$ \ 7) Are you currently: a) Married, b) Living together as married, c) Divorced, d) Seperated, e) Widowed, f) Single.

$\bullet$ \ 8)  Have you had any children? a) No children, b) One child, c) Two children, d) Three children, e) More than three children.

$\bullet$ \ 9)  What is your religious affiliation? a) Potestant, b) Catholic, c) Jewish, d) Islamic, e) Buddhism, f) Others, g) None.

\vspace{0.3cm}

{\large \textbf{Instruction (1/3)}} [\textit{New page}]

In this experiment, you will assess the chances that certain events will happen.

Here is a simple example to explain. Suppose we flip a fair coin, with 50\% chance landing Heads and 50\% chance landing Tails.

$\bullet$ \ What is the probability (\%) that the coin lands Heads? (answer it using an integer between 0 and 100):

Click on the [Submit] button after you finish the answer. Please notice that you can NOT change your answer after submission.

\vspace{0.3cm}

{\large \textbf{Instruction (2/3)}} [\textit{New page}]

\textbf{Overview}

In this coin flipping example, the chance that the coin lands Heads is 50\% and the chance it lands Tails is 50\%.

We will pay you based on your answer. Our payment scheme guarantees that it is always in your best interest to report your truthful assessment of the chance.

\textbf{Payment Details}

In every probability assessment question as above, you will submit X about the chance that an event happens. In our coin flipping example, X represents the percentage chance of Heads. Then computer will randomly draw a value Y from 0 to 100.

If Y is greater than or equal to X, you will win \$6 with Y\% chance. If Y is less than X, you will win \$6 if the event occurs (in this example, the event refers that coin flip lands Heads). Namely,
\[ 
\begin{cases}\$ 6 \text { with } Y \% \text { chance, } & \text { if } Y \geq X ; \\
\$ 6 \text { when the event occurs,} & \text { if } Y<X.\end{cases}
 \]

Given this scheme, it is always in your best interest to choose X that represents your truthful assessment of the chance that the relevant event happens. Thus, the optimal choice of the above example is X = 50.

Click here to see the explanation. [\textit{The following content on this page will only be displayed when this sentence is clicked.}]

\textbf{Explanation}

Consider you submit a lower value for X; for example, X = 20. If the drawn number Y is between 20 and 50, you will win \$6 with Y\% chance. If you had instead submitted X = 50, you are more likely to get the \$6 with 50\% (the coin has landed Heads with 50\% chance).

Similarly, consider choosing a higher value for X; for example, X = 80. If Y is between 50 and 80, you will win \$6 with 50\%, which is the probability of Heads. If you had instead submitted X = 50, you will get the \$6 with Y\% chance, which is between 50\% and 80\%.

\vspace{0.3cm}

{\large \textbf{Instruction (3/3)}} [\textit{New page}]

In our scenario, there are two boxes of balls. Each box has a total of 100 balls, and each ball can be either red or blue. One box is selected, and one ball is randomly drawn as a signal from the selected box. But the selected box is not revealed.

We will tell you the CASE about:

\quad 1) the proportion of two types of balls in the left and right boxes,

\quad 2) the probability of selecting the two boxes.

We will ask you to assess the probability of selecting the left box conditional on the color of the drawn ball.

\textbf{Example}

Here is a sample CASE.

\quad Left box contains 40 red balls and 60 blue balls.

\quad Right box contains 30 red balls and 70 blue balls.

\quad The probability of selecting the left box is 20\%, and the probability of selecting the right box is 80\%.

In this case, one box is selected according to the probability, and one ball is randomly drawn from the selected box.
You will answer the following two questions in two rounds:

$\bullet$ \ If the ball is red, what is the probability (\%) that it comes from the left box?

$\bullet$ \ If the ball is blue, what is the probability (\%) that it comes from the left box?

We have 30 cases, and 2 questions for each, resulting in 60 rounds in total.

The computer will randomly choose one case from the 30 cases. In the chosen case, the computer will first select a box according to the probability, and then randomly draw a ball from the selected box. If the drawn ball is red/blue, we will use your submitted choice for the corresponding question and pay you as explained before. Remember to provide your truthful assessment to maximize the chance of winning \$6.

If you want to check the payment scheme again, please click here. [\textit{The following content before understanding testing questions will only be displayed when this sentence is clicked.}]

You will submit X about the chance that an event happens. Then computer will randomly draw a value Y from 0 to 100.

If Y is greater than or equal to X, you will win \$6 with Y\% chance. If Y is less than X, you will win \$6 if the event occurs (in this example, the event refers that coin flip lands Heads). Namely,
\[ 
\begin{cases}\$ 6 \text { with } Y \% \text { chance, } & \text { if } Y \geq X ; \\
\$ 6 \text { when the event occurs,} & \text { if } Y<X.\end{cases}
 \]

\textbf{Understanding Testing Questions}

$\bullet$ \ 1) In the above example, if the computer draws a blue ball, your answer of which question will be used to pay? a) the question where the ball is red, b) the question where the ball is blue, c) I don’t know.

$\bullet$ \ 2) Suppose you estimate the probability at x\%, which answer will maximize your chance of winning \$6? a) some number smaller than x, b) some number larger than x, c) x, d) I don’t know.

From now on, you probably have a good sense of question. The following questions only vary in terms of the boxes' composition and the selection probability of each box.

Please make sure you understand the rule.

If you are ready to enter the formal task, please click on the following button.

\vspace{0.3cm}

{\large \textbf{Formal Task}} [\textit{New page, 60 rounds in totals.}]

Round \{$n$\}/60: [\textit{$n$ represents the number of round.}]

\quad Left box contains $\{p_l*100\}$ red balls and $\{(1-p_l)*100\}$ blue balls.

\quad Right box contains $\{p_r*100\}$ red balls and $\{(1-p_r)*100\}$ blue balls.

\quad The probability of selecting the left box is $\{\mu*100\%\}$, and the probability of selecting the right box is $\{(1-\mu)*100\%\}$. [\textit{$p_l,p_r,\mu$ are defined in subsection \ref{section:design}.}]

One box is selected according to the probability, and one ball is randomly drawn from the selected box.

$\bullet$ \ If the ball is red, what is the probability (\%) that it comes from the left box?

(answer it using an integer between 0 and 100)

\vspace{0.3cm}

{\large \textbf{Additional Question}} [\textit{New page}]

$\bullet$ \ How do you determine your answer in the previous formal task?

\vspace{0.3cm}

{\large \textbf{The End}} [\textit{New page}]

Thanks for your participation! You have finished the formal task, and gain the participation fee of \$5.5.

The selected round for you is \{$n$\}. [\textit{$n$ represents the selected round for payment.}]

Upon calculation, your bonus is \${$q$}, The bonus will be distributed in a week. [\textit{Based on the payoff calculation, $q$ is set to either 6 or 0.}]

NOTICE: 

Please click the following link to redirect to the Prolific as a final step to update your completion status.

\{ Link \} [\textit{Here we present our redirect URL to Prolific.}]

(If the link does not work, you can copy and paste it into your browser.)

After redirecting, the whole study is finished and you can close this webpage. Thanks for your participation again.

%% file: additional_analysis.tex
\newpage

\section{Additional Table}\label{appendix:analyses}

\setlength{\tabcolsep}{3.5pt}
\begin{table}[h]
\begin{center}
\footnotesize
\begin{threeparttable}   
\begin{tabular}{cccccccccccc}
\hline \hline
 & \multicolumn{9}{c}{$\hat{\lambda}$-base rate balancing aggregator $f^{\hat{\lambda}}_{ap}$ ($\hat{\lambda} < 1$)} & \multirow{2}{*}{\makecell{Average \\ prior}} & \multirow{2}{*}{\makecell{Simple \\ average}} \\
  \cline{2-10}
 \multicolumn{1}{l}{} & $\hat{\lambda}=0.1$ & $\hat{\lambda}=0.2$ & $\hat{\lambda}=0.3$ & $\hat{\lambda}=0.4$ & $\hat{\lambda}=0.5$ & $\hat{\lambda}=0.6$ & $\hat{\lambda}=0.7$ & $\hat{\lambda}=0.8$ & $\hat{\lambda}=0.9$ &  &  \\
 \hline
 Avg. loss & 0.0225 & 0.0205 & 0.0184 & 0.0163 & 0.0141 & 0.0120 & 0.0100 & 0.0083 & 0.0073 & 0.0076 & 0.0091 \\
 Max. loss & 0.0488 & 0.0450 & 0.0411 & 0.0371 & 0.0330 & 0.0289 & 0.0254 & 0.0237 & 0.0221 & 0.0210 & 0.0400 \\
 Diff. of loss & 0.0657 & 0.0648 & 0.0639 & 0.0630 & 0.0621 & 0.0611 & 0.0602 & 0.0592 & 0.0579 & 0.0562 & 0.0536 \\
 \% at info. struct.  & 0.71 & 0.49 & 0.34 & 0.22 & 0.14 & 0.07 & 0.03 & 0.01 & 0.00 & 0.00 & 0.00 \\
  \% at report pair & 25.20 & 24.25 & 23.00 & 21.53 & 19.66 & 17.52 & 14.69 & 11.10 & 7.36 & 4.80 & 3.19 \\
\hline \hline
\end{tabular}
\begin{tablenotes}[flushleft]
\footnotesize \item \textit{Notes:} The number of observations is 29,889 for the first four rows, and 4,218,968 for the row of \% at report pair. For the convenience of comparison, we exclude 44 pairs of subjects' predictions that cannot be aggregated by either average prior aggregator or $\hat{\lambda}$-base rate balancing aggregators ($\hat{\lambda} < 1$). This exclusion applies to cases where, for instance, one subject reports a probability of 0\% while another reports 100\%. Avg. loss and Max. loss represent the average and maximum relative loss when aggregating Bayesian posteriors. Diff. of loss represents the difference of relative loss when aggregating subjects' predictions (Table~\ref{table:agg perf}) and Bayesian posteriors. \% at info. struct. and \% at report pair refer to the percentage proportion of achieving lower loss for subjects' predictions than Bayesian posteriors.

\end{tablenotes}
\end{threeparttable}
\caption{Summary of Aggregators' Performance on Bayesian Posteriors}
\label{table:agg perf appendix}
\end{center}
\end{table}